\documentclass[10pt,journal,compsoc]{IEEEtran}
\usepackage{amsmath, amsfonts, amssymb}
\usepackage{mathtools}
\usepackage{amsthm}
\usepackage{bm}
\usepackage{nicefrac}

\usepackage{algorithm}
\usepackage{algpseudocode}

\usepackage{array}
\usepackage{booktabs}
\usepackage{multirow}
\usepackage{makecell}
\usepackage{threeparttable}
\usepackage{diagbox}
\usepackage{hhline}
\usepackage{colortbl}

\usepackage{graphicx}
\usepackage{wrapfig}
\usepackage[caption=false,font=normalsize,labelfont=sf,textfont=sf]{subfig}

\usepackage{xcolor}
\usepackage{tcolorbox}

\usepackage{textcomp}
\usepackage{url}
\usepackage{verbatim}
\usepackage{xspace}
\usepackage{multicol}
\usepackage{stfloats}

\usepackage{listings}

\usepackage{pifont}

\usepackage{cite}

\newcommand{\ours}[0]{{{RTA}}\xspace}
\newcommand{\code}[0]{\url{https://github.com/EdisonLeeeee/RTA}\xspace}

\definecolor{commentcolor}{RGB}{110,154,155}   

\newcommand{\R}[1]{{%
    \textbf{%
        \ifstrequal{#1}{1}{\textcolor{red}{R#1}}{%
        \ifstrequal{#1}{2}{\textcolor{blue}{R#1}}{%
        \ifstrequal{#1}{3}{\textcolor{magenta}{R#1}}{%
        \ifstrequal{#1}{4}{\textcolor{teal}{R#1}}{%
                           \textcolor{cyan}{R#1}%
        }}}}%
    }%
}}
\definecolor{BlueGreen}{rgb}{0.0, 0.87, 0.87}
\definecolor{WildStrawberry}{rgb}{1.0, 0.26, 0.64}
\newtcbox{\hlprimarytab}{on line, box align=base, colback=BlueGreen!20, size=fbox, before upper=\strut, top=-1.5pt, bottom=-1.5pt, left=-1pt, right=-1pt, boxrule=0pt}
\newtcbox{\hlsecondarytab}{on line, box align=base, colback=WildStrawberry!10,size=fbox, before upper=\strut, top=-1.5pt, bottom=-1.5pt, left=-2pt, right=-2pt, boxrule=0pt}

\definecolor{gold}{HTML}{FFD700}
\definecolor{brown}{HTML}{C0C0C0}
\definecolor{tableheadcolor}{RGB}{255, 255, 255}
\definecolor{tablerowcolor}{RGB}{255, 255, 255}
\definecolor{methodrowcolor}{RGB}{255, 255, 255}

\newcommand{\first}[1]{{\textbf{ {#1}}}} 
\newcommand{\second}[1]{{{\underline{#1}}}}

\newtheorem{theorem}{Theorem}

\newtheorem{assumption}{Assumption}
\newtheorem{remark}{Remark}

\usepackage[capitalize,noabbrev]{cleveref}

\begin{document}

\title{Rethinking Message Passing as Retrieval for Text-Attributed Graph Learning}


\author{
Jintang~Li,
Yuhong~Chen,
Ruofan~Wu,
Binli~Luo,
Jiayi~Ji,
Hui~Li,
and Rongrong~Ji
\IEEEcompsocitemizethanks{
\IEEEcompsocthanksitem Jintang Li, Yuhong Chen, Jiayi Ji, Hui Li, Rongrong Ji are with the Key Laboratory of Multimedia Trusted Perception and Efficient Computing, Ministry of Education of China, Xiamen University,
Xiamen 361005, China (E-mail: edisonlee@xmu.edu.cn, yuhongchen@stu.xmu.edu.cn, jjyxmu@gmail.com, hui@xmu.edu.cn, rrji@xmu.edu.cn).
\IEEEcompsocthanksitem Ruofan Wu is with the School of Statistics of Fudan University, Shanghai 200433, China (E-mail: wuruofan1989@gmail.com).
\IEEEcompsocthanksitem Binli Luo is with the School of Computer Science and Engineering of Central South University, Changsha 410000, China (E-mail: lblhandsome@gmail.com).
}
}

\markboth{Journal of \LaTeX\ Class Files,~Vol.~14, No.~8, August~2021}%
{Shell \MakeLowercase{\textit{et al.}}: A Sample Article Using IEEEtran.cls for IEEE Journals}


\IEEEtitleabstractindextext{%
  \begin{abstract}
    Graph neural networks (GNNs) are typically conceptualized as message-passing neural networks, yet it remains unclear why neighborhood aggregation reliably outperforms node-wise multilayer perceptrons (MLPs). 
    Despite its empirical success, this paradigm can be computationally expensive and sensitive to imperfect graph structures. In this work, we present a retrieval-augmented view of GNNs: each layer makes predictions by applying an MLP to a node representation together with a permutation-invariant summary of retrieved graph context. Motivated by this perspective, we propose \ours, a simple MLP-based framework that replaces structural message passing with label-aware retrieval and propagation. We provide theoretical insights that (i) connect retrieval-based aggregation to softmax-attention message passing, and (ii) establish the robustness of retrieved-context supervision to mis-retrieved outliers. Experiments on multiple text-attributed graph benchmarks show that \ours matches or even outperforms strong GNN and graph LLM baselines while improving efficiency and robustness across diverse scenarios.
  \end{abstract}

}

\maketitle
\IEEEdisplaynontitleabstractindextext
\IEEEpeerreviewmaketitle

\IEEEraisesectionheading{\section{Introduction}\label{sec:introduction}}

Graph neural networks (GNNs) have attracted substantial attention due to their strong empirical performance on graph-structured data~\cite{gcn,gat,graphsage,maskgae}, with applications ranging from social networks~\cite{Fan0LHZTY19} and molecular chemistry~\cite{SunDY22} to recommendation systems~\cite{lightgcn}. The dominant paradigm behind modern GNNs is \emph{message passing}~\cite{message_passing}: nodes iteratively exchange information along edges and update their representations by aggregating neighbor messages through permutation-invariant operators (e.g., sum or mean). While this mechanism is effective in many settings, it also introduces well-known challenges, including limited scalability, sensitivity to graph structure, degraded performance on heterophilous graphs~\cite{hetero2net}, and representation collapse due to oversmoothing and oversquashing in deep propagation~\cite{bottleneck}. More fundamentally, although message passing has become the dominant paradigm for GNNs, it remains unclear whether its effectiveness stems from graph-based propagation itself or, more generally, from conditioning node predictions on informative contexts.

In parallel, retrieval-augmented generation (RAG) has emerged as a powerful paradigm in natural language understanding~\cite{rag_survey}. RAG~\cite{rag} follows a \emph{retrieve-then-reason} procedure: given a query, it retrieves a small set of relevant contexts and conditions a parametric model (e.g., an LLM) on the retrieved evidence to produce an output. This mechanism reduces reliance on parametric memory and improves robustness by grounding predictions in external information.
\begin{figure}[t]
  \centering
  \includegraphics[width=\linewidth]{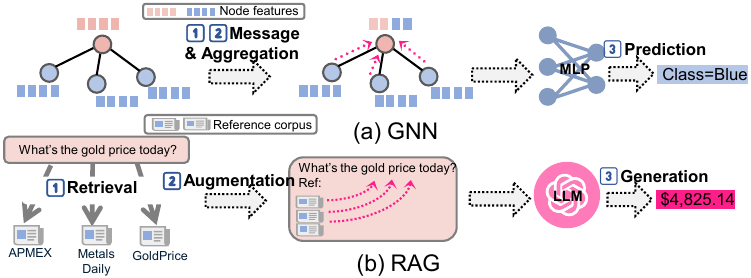}
  \caption{Technical comparison between {GNN} and {RAG}. Both leverage contextual information beyond the raw input.}
  \label{fig:example}
  \vspace{-5mm}
\end{figure}

Despite their different application domains, GNNs and RAG share a common computational pattern: both condition predictions on \emph{contextual information} beyond the raw input (see Figure~\ref{fig:example}). In a message-passing layer, the context is defined by the node's neighborhood, whereas in RAG it consists of a set of retrieved documents or passages. However, this shared computational structure has not been explicitly characterized from a unified perspective, leaving unclear how context selection, aggregation, and prediction in message passing relate to retrieval-augmented inference.

In this work, we develop a retrieval-augmented view of a broad family of message-passing GNNs:
\begin{center}
  \textit{A message-passing layer can be interpreted as an MLP operating on a node and a retrieved context set.}
\end{center}
Concretely, we show that the message passing step of GNNs can be cast as a retrieval and augmentation module, and that different GNN variants correspond to different choices of \emph{retriever} (which contexts are selected) and \emph{set function} (how contexts are aggregated). This perspective clarifies why conditioning on neighbors can generally outperform node-independent models: the model is effectively performing prediction with retrieved evidence, where the evidence is provided by graph-induced or learned retrieval sets.

Building on this connection, we propose \textbf{\ours} (\textbf{\underline{R}}etrieve-\textbf{\underline{T}}hen-\textbf{\underline{A}}ggregate), a simple retrieval-augmented framework for text-attributed graph representation learning that avoids explicit multi-hop message passing. Crucially, \ours supports label-aware retrieval, where labeled nodes provide necessary supervision signals that guide both retrieval and aggregation without information leakage. \ours proceeds in three stages: (i) a label-aware contextual encoder produces node embeddings by augmenting each node's text with a small set of structural neighbor texts; (ii) a top-$k$ retriever in the structure and embedding space constructs a \emph{retrieval graph} that can connect semantically related nodes beyond local edges; and (iii) an MLP-style update aggregates retrieved neighbors to produce task-specific representations. By replacing recursive message passing on the original graph with localized retrieval-driven interactions, \ours achieves better performance while improving scalability and robustness across diverse settings.


Theoretically, we establish two results that characterize retrieval-based message passing and directly align with our method design. First, we show that top-$k$ retrieval followed by aggregation provides a sparse approximation of softmax-attention message passing, where the approximation error is controlled by the retrieval margin and the score distribution within the retrieved context set. Second, we prove that label-aware retrieved-context supervision reduces the gradient influence of mis-retrieved outliers, providing a theoretical explanation for the robustness of retrieval-based aggregation under noisy contexts. Overall, both results justify retrieval as a principled mechanism for constructing task-relevant context sets and provide theoretical support for label-guided retrieval and aggregation.

The main contributions of our work are as follows:
\begin{itemize}
  \item \textbf{Formulation.} We reveal an intrinsic duality between message-passing aggregation and retrieval-augmented MLPs, offering a unified view of GNN layers as prediction over retrieved contexts.
  \item \textbf{Framework.} Leveraging this connection, we propose \ours, a scalable retrieval-centric alternative to message passing that integrates label-aware contextual encoding, embedding and structural space retrieval, and label-enhanced aggregation.  
  \item \textbf{Theoretical results.} We establish theoretical connections between retrieval-based aggregation and message passing, and prove optimization robustness under retrieved-set supervision.
  We present theoretical guarantees for (i) the equivalence of retrieval-based message passing and (ii) optimization robustness under retrieved-set supervision.
  \item \textbf{Empirical results.} We conduct extensive experiments on multiple text-attributed graph benchmarks, showing that \ours matches or outperforms strong GNN baselines with improved efficiency and robustness.
\end{itemize}
Overall, our work takes a step toward demystifying message passing by recasting it as retrieval over contextual evidence, and it suggests a practical design space in which graph representation learning can be carried out by retrieval-defined computation graphs rather than fixed-edge diffusion.

\section{Related work}
Our work connects three lines of research: message-passing graph neural networks, learning over text-attributed graphs with language models, and retrieval-augmented generation. 
While these areas have been extensively studied in isolation, their underlying computational connection remains underexplored. 
We review them below and highlight how our work differs by interpreting graph representation learning itself as a retrieval-augmented prediction process.

\subsection{Graph neural networks}
GNNs are a flexible class of models for learning representations over graph-structured data, which represent a significant leap forward in the ability to model and analyze complex relationships within data. By iteratively aggregating and transforming information from a node's local neighborhood, GNNs learn powerful, context-aware representations that are crucial for making accurate predictions and uncovering hidden patterns within the graphs \cite{gcn,gat,graphsage}.
Beyond conventional local message passing, graph Transformers have emerged as an alternative paradigm for capturing long-range dependencies through attention-based interactions. Representative methods include NodeFormer~\cite{nodeformer}, which develops a scalable all-pair attention mechanism for graph representation learning, and DGT~\cite{DGT}, which dynamically selects structurally and semantically relevant nodes to enable efficient sparse attention.
The research topic of GNNs is vibrant and rapidly evolving, encompassing a wide range of theoretical advancements and algorithmic innovations~\cite{graphsaint,mediangcn,gin,maskgae}.

\subsection{Learning over text-attributed graphs}
In real-world graph tasks, nodes often have textual attributes that convey richer information, resulting in text-attributed graphs (TAGs)~\cite{tape}.
TAG provides a scenario for applying LLMs on graphs that store information in natural language. Combining GNNs with LLMs can leverage the robust textual understanding of LLMs while harnessing GNNs' ability to capture structural relationships. Specifically, there are several lines of research according to the role of LLMs on graphs, including LLM as Predictor~\cite{graphgpt,ChenMLJWWWYFLT23,bundle}, LLM as Aligner~\cite{relm,glem,wang2025generalization} and LLM as Enhancer~\cite{tape,giant,ofa}. While adapting LLMs to graphs has emerged as a research focus, little effort has been dedicated to uncovering the correlations between GNNs and language modeling.
A closely related work is AuGLM, which treats retrieval as input augmentation for an LLM classifier, whereas \ours studies retrieval as the context-selection primitive underlying graph message passing.

\subsection{Retrieval-augmented generation}
Retrieval-augmented generation~\cite{rag} is an established approach to enhance the generation quality of LLMs by leveraging external knowledge bases.
RAG achieves good performance on various tasks such as language modeling~\cite{speculative_rag}  and question-answering~\cite{IzacardLLHPSDJRG23}, without additionally fine-tuning LLMs.
To improve generation, several studies have been dedicated to enhancing the capability of RAGs with correction~\cite{abs-2401-15884}, verification~\cite{LLatrieval}, or critique~\cite{selfrag}; see~\cite{rag_survey} for a comprehensive survey on RAG techniques.
Recently, researchers and practitioners have started to explore the role of graphs in facilitating RAG systems, with the most promising works being GraphRAG~\cite{graphrag} and its variants~\cite{gretriever,hipporag,hipporag2}.
GraphRAG-based approaches leverage knowledge graphs as a structured source of context and factual grounding for LLMs. By supplying entity-centric information such as textual descriptions, they provide richer, more organized evidence, which can encourage deeper, relation-aware reasoning by the LLM.
Note that GraphRAG-based approaches and our work represent two orthogonal lines of research. GraphRAG-based approaches enhance LLM reasoning by incorporating graph-structured knowledge into RAG. In contrast, \ours advances graph representation learning by leveraging recent advances in RAG.

\section{Preliminary}
We first introduce the notations and background needed to connect message-passing GNNs with retrieval-augmented generation. 
The goal of this section is not to review all variants, but to expose the shared computational pattern behind the two paradigms.
\subsection{Notations}
With graph-structured inputs, we typically assume a graph $\mathcal{G} = (\mathcal{V}, \mathcal{E})$, where $\mathcal{E} \subseteq \mathcal{V} \times \mathcal{V}$ specifies pairs of nodes in $\mathcal{V}$ that are connected. To unleash the power of LLMs on graphs, we consider a more challenging scenario, i.e., text-attributed graphs, which provide rich semantic information for the representation learning.
In the context of text-attributed graphs, we are given a set of textual features $\mathcal{T} = \{t_1, t_2, \ldots, t_N\}$, where each $t_u$ is the textual corpus associated with node $u \in \mathcal{V}$. In such scenarios, the text set $\mathcal{T}$ is encoded as the node feature matrix $\mathbf{X} = \{\mathbf{x}_1, \ldots, \mathbf{x}_N\}$.

\subsection{Message-passing graph neural networks}
Given a graph $\mathcal{G}$ and $d$-dimensional node features $\mathbf{X}$, message-passing GNNs aim to obtain informative node representations capturing both structural properties and connections between node features.
\begin{equation}
  \label{eq:message_passing}
  \mathbf{x}_u^{\prime}=\phi\left(\mathbf{x}_u, \bigoplus   \left(\{\mathbf{x}_v: {v \in \mathcal{N}_u}\}\right)\right),\  \mathcal{N}_u=\{v: (u,v)\in \mathcal{E}\},
\end{equation}
with the core components including a permutation invariant aggregation function $\bigoplus$ (e.g., mean or sum) and a combine or update function $\phi$; $\mathcal{N}_u \subseteq \mathcal{V}$ is a set of nodes adjacent to node $u$.
Specifically, node $u$ aggregates the messages along incoming edges using the aggregation function $\bigoplus$. Then, the update function $\phi$ updates the representations of the receiver node $u$ based on its current representation ($\mathbf{x}_u$) and the aggregated messages.
The update function $\phi$ allows the node to learn how to integrate its own features with the aggregated context effectively.
Typically, $\phi$ is an MLP network and the design of the function $\phi$ is what mostly distinguishes one type of GNN from the other.
For example, GraphSAGE~\cite{graphsage} implements GNN as:
\begin{equation}
  \label{eq:graphsage}
  \mathbf{x}_u^{\prime}=\operatorname{MLP}\left(\mathbf{x}_u, \operatorname{MEAN}   \left(\{\mathbf{x}_v: {v \in \mathcal{N}_u}\}\right)\right).
\end{equation}
In recent years, GNNs have been extended with several advanced variants that incorporate additional components, such as attention mechanisms~\cite{gat}, advanced aggregation functions~\cite{mediangcn}, and residual connections~\cite{revgat}.

\subsection{Retrieval-augmented generation}
RAG is a natural language querying approach for enhancing off-the-shelf LLMs with external knowledge~\cite{rag}. In an RAG application, we are given a collection of documents or contexts (e.g. Wikipedia), which provides the grounded knowledge and serves as the external knowledge base. Given a user query $\mathbf{q}$ and a corpus $\mathbf{C}$, a dense embedding based retriever~\cite{DPR} first retrieves top-$k$ contexts $\mathbf{C_k}=\{\mathbf{c}_1,\ldots,\mathbf{c}_k\}$ that are most relevant to the query $\mathbf{q}$. Here $\mathcal{R}_k\subseteq \mathbf{C}$ is a set of retrieved contexts from $\textbf{C}$. Then, LLM reads the top-$k$ contexts along with the query $\mathbf{q}$  to generate the answer $\mathbf{r}$.
Formally, a standard RAG-based system can be formulated as follows:
\begin{equation}
  \begin{aligned}
    \mathbf{r}    & =\operatorname{Reason}\left(\mathbf{q},\mathcal{R}_k\right),                                                \\
    \mathcal{R}_k & =\left\{\mathbf{c}_{{1}}, \cdots, \mathbf{c}_{{K}}\right\}=\operatorname{Retrieval}(\mathbf{q} \mid \mathbf{C}).
  \end{aligned}
\end{equation}
In practice, advanced RAG may incorporate chunking and reranking operations, which segment the contexts into chunks for improved retrieval and rerank the retrieved results for enhanced generation~\cite{flashrag}. In this work, we consider a classical RAG system that comprises only a retriever and a generator, for simplicity.

\section{Present Work: \ours}
\label{sec:method}
In this section, we first present our perspective on how GNNs relate to the RAG-based framework. Next, we introduce \ours, an RAG-inspired graph learning framework designed to learn representations on text-attributed graphs using a pure RAG scheme. \ours leverages valuable insights from GNNs and label propagation~\cite{lp}, a label-specific variant of message-passing GNNs, to efficiently and effectively capture graph context. 

\begin{figure}[t]
  \centering
  \includegraphics[width=\linewidth]{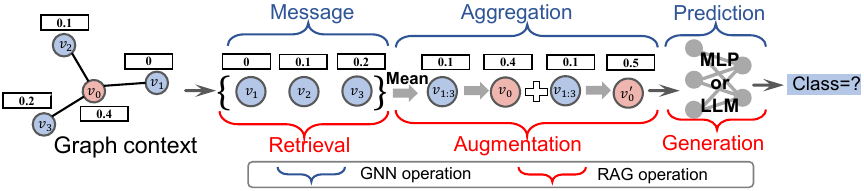}
  \caption{Illustrative example on how GNNs are decomposed into RAG-like architectures.}
  \label{fig:gnn_as_rag}
\end{figure}

\subsection{Connections between GNN and RAG}
\subsubsection{Unified view}
GNNs and RAG are important tools for graph structure learning and natural language understanding, respectively. Both approaches leverage similar principles by incorporating \textit{additional context} beyond the raw input to enhance comprehension. In the case of GNNs, the additional context is derived from aggregating information from neighboring nodes and the graph structure itself. For RAGs, context is incorporated by retrieving relevant knowledge from external sources to supplement the input. This shared principle motivates us to align them within a unified formulation.
In this regard, we formulate {GNNs} and {RAG} that follow similar rules as below:
\begin{align}
  \label{eq:gnn2}
   & \mathbf{x}_u^{\prime}=\underbrace{{\phi}\underbrace{\left(\mathbf{x}_u, \bigoplus (\{{\mathbf{x}_v : v \in \mathcal{N}_u}\})\right)}_\text{Aggregation}}_\text{Prediction},\ \underbrace{\mathcal{N}_u=\{v: (u,v)\in \mathcal{E}\}}_\text{Message},                                                 \\
  \label{eq:rag}
   & {\mathbf{r}=\underbrace{{\phi}\underbrace{\left(\mathbf{q}, \bigoplus (\{\mathbf{c}_j: {\mathbf{c}_j \in \mathcal{R}_k}\})\right)}_\text{Augmentation}}_\text{Generation},\ \underbrace{\mathcal{R}_k(\mathbf{q})=\mathrm{Top}\text{-}k_{\mathbf{c}\in\mathcal{R}}\ \langle\mathbf{q},\mathbf{c}\rangle}_{\text{Retrieval}}},
\end{align}
where the cosine similarity is defined as $\langle x,y\rangle=\frac{x\cdot y}{\|x\|\|y\|}$. Here, we omit the text embedding step for the sake of simplicity.
In RAG, $\mathcal{R}_k(\mathbf{q})$ denotes the top-$k$ retrieved corpus from $\mathcal{R}$ that are most relevant to the query $\mathbf{q}$ according to a dense retriever $\psi$, and $\bigoplus$ is a combination operator over the retrieved contexts, such as concatenation. We overload notations in Eq.~\eqref{eq:rag} for simplicity and omit prompt details.

\subsubsection{Label propagation as label-centric RAG}
This unified view suggests that the ``context'' set in Eq.~\eqref{eq:gnn2}--\eqref{eq:rag} need not be restricted to feature vectors or text tokens; it can also contain \emph{label evidence} from a labeled set (e.g., human supervision).
In particular, classical label propagation (LP)~\cite{lp} can be interpreted as a label-specific message-passing procedure:
each node \emph{collects} (retrieves) label signals from its neighbors and \emph{mixes} (augments) them through a propagation operator. Under the unified view above, LP corresponds to a special case where the retrieval set is the graph neighborhood $\mathcal{N}_u$, the retrieved ``contexts'' are label message vectors, and the ``generation'' step reduces to a simple readout on the aggregated label context:
\begin{equation}
\begin{aligned}
\label{eq:lp}
 & \mathbf{p}_u^\prime
 =\underbrace{\phi\underbrace{\left(\mathbf{p}_u,\bigoplus\big(\{\mathbf{p}_v: v\in\mathcal{N}_u\}\big)\right)}_{\text{Propagation}}}_{\text{Inference}},
 \
 \underbrace{\mathcal{N}_u=\{v:(u,v)\in\mathcal{E}\}}_{\text{Message / Retrieval}},\\
 & \mathbf{p}_v=\mathbb{I}[v\in\mathcal{T}]\,\omega(\tilde y_v)+\mathbb{I}[v\notin\mathcal{T}]\,\mathbf{0},
 \
 \hat y_u=\arg\max_{c\in[C]}[\mathbf{p}_u^\prime]_c,
\end{aligned}
\end{equation}
where $\mathcal{T}$ denotes the labeled nodes (training set), $\omega(\tilde y_v)$ can be a one-hot label embedding, and unlabeled nodes contribute a padding message (e.g., $\mathbf{0}$). This perspective provides a principled motivation for \emph{label-aware RAG on graphs}: when partial labels are available, they can be treated as additional retrieved context to guide both aggregation and retrieval.

\subsubsection{Differences and design implications}
The unified formulation in Eqs.~\eqref{eq:gnn2}, \eqref{eq:rag} and \eqref{eq:lp} highlights a shared \emph{retrieve-and-aggregate} skeleton, but the three paradigms differ in how the context set is formed and how it is fused.
(i) \textbf{context selection}: in GNNs/LP, the graph structure fixes the retrieved set $\mathcal{N}_u$, yielding query-dependent and degree-varying neighborhoods; in RAG, a learned retriever returns a fixed-size top-$k$ set $\mathcal{R}_k(\mathbf{q})$, which effectively induces a content-dependent similarity graph with controlled out-degree.
(ii) \textbf{context fusion}: RAG typically augments inputs via prompt construction (e.g., concatenating retrieved texts),
whereas GNNs employ explicit permutation-invariant set operators (mean/sum/max/attention) over vector states; LP further
exposes a label-only extreme where the fused context consists purely of propagated supervision signals.
These differences suggest a natural hybrid for graphs: learn the context set via embedding and structural space top-$k$ retrieval while retaining efficient set aggregation, and treat available labels as additional retrieved context.
We next instantiate this idea by introducing label embeddings into both contextual encoding and retrieved-neighbor aggregation.

\begin{figure*}[t]
  \centering
  \includegraphics[width=0.82\linewidth]{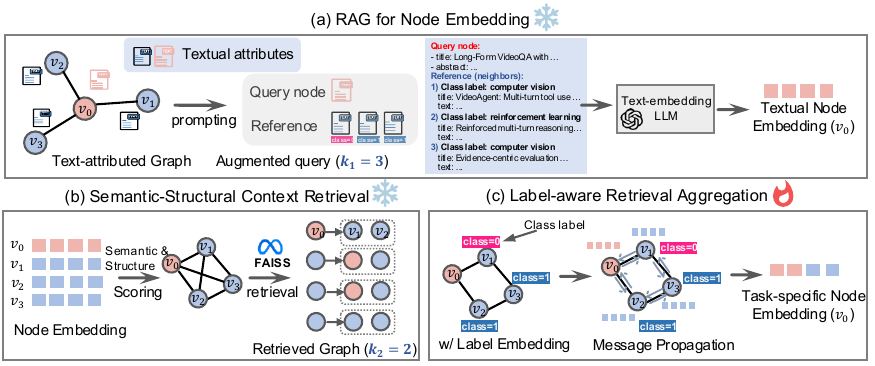}
  \caption{
    Overall framework of \ours. \ours recasts graph representation learning as a retrieval-augmented pipeline that avoids explicit multi-hop message passing over the original graph. Specifically, \textbf{(a)} for a query node $v_0$, we augment its text with a small set of structural neighbors (size $k_1$) and corresponding labels as the resulting prompt to obtain the contextual node embedding $\mathbf{x}_{v_0}$. Then, \textbf{(b)} based on the learned embeddings and the original graph structure, we construct a retrieval-defined context graph by jointly considering semantic similarity and structural proximity. Specifically, semantic retrieval is performed in the embedding space (e.g., via FAISS), while structural retrieval exploits graph diffusion information. Finally, \textbf{(c)} we aggregate features and available label embeddings from retrieved contexts and apply an MLP-style update to produce task-specific node representations for prediction.
}
  \label{fig:framework}
\end{figure*}

\subsection{\ours: RAG-based graph learning framework}

By drawing parallels between GNN and RAG, we now explore how integrating retrieval-augmented mechanisms into a simple MLP network can efficiently reshape graph representation learning. The overall framework of \ours is shown in Figure~\ref{fig:framework}. Different from conventional message-passing GNNs that rely on static edges, \ours retrieves and aggregates semantically relevant contexts from the embedding and structural space, enabling adaptive graph construction and scalable structure learning.

\subsubsection{Label-aware retrieval-augmented node embedding}
Existing learning paradigms on TAGs typically require a text encoder, such as BERT~\cite{bert} or DeBERTa~\cite{deberta}, to obtain node feature embeddings from associated text attributes. For example, a language model $\operatorname{LLM}(\cdot)$ processes the text attribute $t_u$ to generate $\mathbf{x}_u=\operatorname{LLM}(t_u)$. However, this operation treats each node in isolation and ignores the rich contextual information residing in neighboring textual attributes, leading to less informative embeddings. 

Inspired by the retrieve-then-reason paradigm in RAG, we enrich the textual representation of each node by augmenting it with contextual text retrieved from its structural neighbors. Since the graph itself provides a localized retrieval source, we can bypass explicit document retrieval and instead treat the neighbor texts as intrinsic contexts:
\begin{equation}
  \mathbf{x}_u = \operatorname{Enc}_\theta\big(t_u,\bigoplus\big(\{t_v: v \in \mathcal{N}_u^{(k_1)}\}\big)\big)\in\mathbb{R}^d,
\end{equation}
where $\mathcal{N}_u^{(k_1)}$ denotes the $k_1$ sampled neighbors of node $u$, and $\bigoplus$ represents a permutation-invariant combination such as prompt concatenation for the neighbor contexts.  Generally, $\operatorname{Enc}_\theta$ is implemented as an LLM backbone for semantic textual embedding.
To further enhance retrieval quality, we incorporate partial label signals into both the embedding and aggregation stages. 
When some node labels are available, we may inject them as additional context, which act as supervision-aware anchors that guide the retrieval process toward semantically and label-consistent neighborhoods:
\begin{equation}
\mathbf{x}_u = \mathrm{Enc}_\theta\!\Big(t_u,\ \bigoplus\big(\{(t_v,y_v): v\in \mathcal{N}_u^{(k_1)}\}\big)\Big),
\label{eq:enc_label}
\end{equation}
where $y_v$ denotes the label of node $v$.  
The introduction of label information is motivated by the label propagation~\cite{lp} algorithm to bring supervision to the message-passing (retrieval) step.
This retrieval-augmented encoding stage allows the LLM to generate context-aware node embeddings that better reflect both semantic and structural correlations.  
Empirically, such contextual augmentation yields richer semantic alignment across neighboring nodes and provides a stronger basis for subsequent structure learning.


\subsubsection{Retrieval-augmented aggregation}
Once the contextual embeddings $\{\mathbf{x}_u\}_{u\in V}$ are obtained, \ours performs graph structure learning by constructing a retrieval-defined context graph.  
Rather than propagating messages along pre-defined edges, we rebuild the graph topology through a retrieval process that identifies task-relevant nodes from both semantic and structural perspectives.  
For each node $u$, we retrieve a set of relevant contexts according to a retrieval score function:
\begin{equation}
\mathcal{R}_u^{(k_2)} 
= 
\mathrm{Top}\text{-}k_{2{v\neq u}} S(u,v),
\label{eq:retrieval}
\end{equation}
where $S(u,v)$ denotes the relevance score between nodes $u$ and $v$. 
Specifically, we consider two complementary retrieval signals. 
The semantic retrieval score is calculated based on the contextual embeddings:
\begin{equation}
S_{\rm sem}(u,v)
=
\frac{\mathbf{x}_u^\top \mathbf{x}_v}
{\|\mathbf{x}_u\|_2\|\mathbf{x}_v\|_2},
\end{equation}
which identifies nodes with similar semantic meanings and can be efficiently implemented using approximate nearest neighbor search, such as FAISS~\cite{faiss}. 
In addition to semantic similarity, we introduce a structural retrieval score to preserve the topological information of the original graph. 
Specifically, we employ personalized PageRank (PPR)~\cite{ppr} to measure the diffusion-based proximity between nodes. 
For a query node $u$, its PPR vector $\boldsymbol{\pi}_u$ is defined as:
\begin{equation}
\boldsymbol{\pi}_u
=
\alpha \mathbf{e}_u
+
(1-\alpha)\mathbf{P}^{\top}\boldsymbol{\pi}_u,
\label{eq:ppr}
\end{equation}
where $\mathbf{P}$ is the transition matrix of the original graph, $\mathbf{e}_u$ is the one-hot personalization vector centered at node $u$, and $\alpha$ is the teleport probability. 
The structural relevance between nodes $u$ and $v$ is then defined as:
\begin{equation}
S_{\rm str}(u,v)=\pi_u(v),
\label{eq:ppr_score}
\end{equation}
where $\pi_u(v)$ denotes the probability mass assigned to node $v$ in the personalized PageRank vector of node $u$.
Since semantic similarity and PPR scores may lie on different numerical scales, we apply query-wise min-max normalization to make the two retrieval signals comparable:
\begin{equation}
\small
\widetilde{S}_{r}(u,v)
=
\frac{
S_{r}(u,v)-\min_{j\neq u}S_{r}(u,j)
}{
\max_{j\neq u}S_{r}(u,j)-\min_{j\neq u}S_{r}(u,j)+\epsilon
},
r\in\{\mathrm{sem},\mathrm{str}\},
\label{eq:score_norm}
\end{equation}
where $\epsilon$ is a small constant for numerical stability.
The final retrieval score combines the normalized semantic and structural evidence:
\begin{equation}
S(u,v)
=
\lambda \widetilde{S}_{\rm sem}(u,v)
+
(1-\lambda)\widetilde{S}_{\rm str}(u,v),
\label{eq:hybrid_score}
\end{equation}
where $\lambda\in[0,1]$ controls the balance between semantic and structural retrieval. 
Since the fused score is used only for top-$k$ ranking, we set $\lambda=0.5$ by default to equally balance the two retrieval signals without dataset-specific tuning.
Optionally, we further symmetrize the retrieved edges by mutual-$k$NN to ensure a stable, undirected retrieval graph.  This process adaptively redefines connectivity according to both semantic relevance and structural proximity, resulting in a retrieval-defined graph that better reflects task-relevant relations.  As illustrated in Figure~\ref{fig:framework}(b), the retrieval-aggregation process enables the model to perform structure learning in a data-driven and scalable manner.

\subsubsection{Label-aware retrieval aggregation}
With the dynamically constructed similarity graph, message passing can be performed in a simplified yet expressive form. 
Instead of iterating over fixed edges, each node aggregates information from its retrieved contexts:
\begin{align}
\mathbf{x}_u^{\prime}=\operatorname{MLP}\left(\mathbf{x}_u, \bigoplus\left(\{\mathbf{x}_v: v \in \mathcal{R}_u^{(k_2)}\} \right)\right),
\end{align}
This design can be viewed as an MLP operating over a retrieval-defined neighborhood.  
It eliminates the need for explicit graph convolutions and avoids common issues such as oversmoothing and computational overhead caused by multi-hop message passing.  
Moreover, the retrieval mechanism allows efficient graph updates through approximate nearest neighbor (ANN) search tools such as FAISS~\cite{faiss}, ensuring scalability to large text-attributed graphs.
We embed labels via $\omega:[C]\to\mathbb{R}^{d}$ (one-hot or learned embedding).
Define $\mathbf{z}_v=\omega(y_v)$ if $v\in \mathcal{T}$ and $\mathbf{z}_v=\mathbf{0}$ otherwise.
We aggregate retrieved neighbors and apply an MLP update:
\begin{align}
\mathbf{x}_u^{\prime}=\operatorname{MLP}\left(\mathbf{x}_u, \bigoplus\left(\{\mathbf{x}_v+\mathbf{z}_v: v \in \mathcal{R}_u^{(k_2)}\} \right)\right),
\end{align}
Unlike transductive label propagation that directly diffuses labels along graph edges, \ours incorporates label embeddings as auxiliary retrieved context and does not directly overwrite unlabeled predictions. Labeled nodes influence which contexts are retrieved, but do not directly overwrite unlabeled representations. 


\subsubsection{Relation to prior work}
While prior efforts such as GraphRAG~\cite{graphrag} and G-Retriever~\cite{gretriever} enhance RAG systems using graph structures, they treat graphs as auxiliary knowledge for retrieval in language models.  
In contrast, \ours reverses this direction: it treats retrieval as the fundamental operation for graph learning itself.  
By shifting from message passing on fixed edges to retrieval-based dynamic connectivity, \ours transforms graph learning into a retrieval-augmented reasoning process.  
This perspective unifies semantic retrieval and structural aggregation, resulting in a scalable, adaptable, and label-aware framework that generalizes beyond traditional GNN design.

\subsection{Algorithm}
\label{sec:alg}
To facilitate a better understanding of the functionality of \ours, we provide PyTorch-style pseudocode for its implementation in Algorithm~\ref{algo:pseudo_rag}.











\begin{algorithm}[H]
\caption{PyTorch-style pseudocode for our model.}
\label{algo:pseudo_rag}

\begin{algorithmic}[1]

\Statex \textcolor{gray}{$\triangleright$ Inputs: text attributes $t$, optional labels $y$, graph $g$}

\Statex \textcolor{gray}{$\triangleright$ L-I: label-aware contextual embedding}
\State $C=\textcolor{blue}{\mathrm{Sample}}(g,k_1)$
\Comment{$C[u]$: structural contexts for node $u$}

\State $X=\textcolor{blue}{\mathrm{LLM}}(t,C,y)$
\Comment{$X=\{x_u\}_{u\in V}$: contextual node embeddings}

\State $Z=w(y)$
\Comment{set $Z[u]=0$ if $y_u$ is unavailable}

\Statex \textcolor{gray}{$\triangleright$ Semantic-structural context retrieval}
\State $S_{\mathrm{sem}}
=\textcolor{blue}{\mathrm{Norm}}\!\left(
\mathrm{CosSim}(X)
\right)$
\Comment{semantic relevance}

\State $S_{\mathrm{str}}
=\textcolor{blue}{\mathrm{Norm}}\!\left(
\mathrm{PPR}(g)
\right)$
\Comment{structural relevance}

\State $S=\lambda S_{\mathrm{sem}}+(1-\lambda)S_{\mathrm{str}}$
\Comment{retrieval score fusion}

\State $N=\textcolor{blue}{\mathrm{TopK}}(S,k_2)$
\Comment{$N[u]=\mathcal{R}_u$: retrieved contexts for node $u$}

\Statex \textcolor{gray}{$\triangleright$ L-II: label-enhanced retrieval aggregation}
\State $M=(X+Z)[N].\textcolor{blue}{\mathrm{mean}}(\mathrm{dim}=1)$
\Comment{aggregate retrieved contexts}

\State $H=X+M$
\Comment{update node representations}

\State $\mathrm{out}=\textcolor{blue}{\mathrm{MLP}}(H)$

\State \textcolor{blue}{\textbf{return}} $\mathrm{out}$

\end{algorithmic}
\end{algorithm}

\section{Theoretical Analysis}
In this section, we are motivated to bridge the gap between RAG and GNNs. We first develop a theoretical connection between message-passing GNNs and RAG by viewing both as instances of prediction over retrieved context sets. This perspective clarifies why message passing improves over node-independent models. Then, we demonstrate the robustness and optimization stability of retrieval-induced graph learning. Due to space limitations, we present simplified statements here and defer the full theorems and proofs to the supplementary material.


\begin{theorem}[Informal]
\label{thm:topk_as_attention_short}
Let $\gamma_u>0$ be the top-$k$ margin of retrieval scores with query node $u$, i.e., the gap between the $k$-th and the $k+1$-th largest score. Further let a GNN oracle be an arbitrary softmax-attention message passing at temperature $\tau$ induced by retrieval scores. Then under the regularity conditions as detailed in the supplementary material, there exists a RAG-style construction that approximates the oracle message at node $u$ with $\ell_\infty$ error $O(\exp(-\gamma_u/\tau))$.
\end{theorem}

\begin{theorem}
\label{thm:tolerance1}
Let $\mathcal{R}_u$ be a retrieved set for some query node $u$, with prediction distribution
$p(\mathcal{R}_u)=(p_1,\ldots,p_C)$ induced by retrieved-context aggregation.
Consider an outlier node $o\in\mathcal{R}_u$ whose individual distribution is $p_o=(p_1',\ldots,p_C')$ and define
$m'=\arg\max_{c\in[C]} p_c'$.
Let $y_u$ be the ground-truth label of node $u$, treated as fixed during differentiation.
If $y_u\neq m'$ and $p_{m'}'\ge p_{m'}$, then
\[
0 \leq \frac{\partial \mathcal{L}(\mathcal{R}_u)}{\partial \ell'_{m'}} \leq
\frac{\partial \mathcal{L}(o)}{\partial \ell'_{m'}},
\]
where $\mathcal{L}(\mathcal{R}_u)=\mathrm{CE}(p(\mathcal{R}_u),y_u)$ and
$\mathcal{L}(o)=\frac{1}{|\mathcal{R}_u|}\mathrm{CE}(p_o,y_u)$, and $\ell'_{m'}$ denotes the outlier logit for class $m'$.
\end{theorem}

\begin{remark}
If the attention scores are constant on the context set, softmax attention reduces to uniform averaging, which matches the mean aggregator used in models such as GCN and GraphSAGE. 
Sum aggregation differs only by a degree-dependent scaling and can be viewed as the same update up to normalization. 
Therefore, the message-level conclusion remains informative for a broad class of message-passing GNNs, even when the final layer implementation does not explicitly use attention.
\end{remark}


\begin{remark}
Theorem~\ref{thm:tolerance1} shows that when retrieval introduces an outlier node $o$ whose most confident class $m'$ conflicts with the ground-truth label $y_u$ and satisfies $p'_{m'}\ge p_{m'}$, the retrieved-context objective $\mathcal{L}(\mathcal{R}_u)$ is more tolerant compared to the individual supervision $\mathcal{L}(o)$, as evidenced by a smaller penalty imposed by the gradient.
\end{remark}

\section{Experiments}
In this section, we present empirical evaluation results of \ours on real-world graph benchmarks.
The source code, including the datasets and all necessary scripts for reproducing the results, is available at \code.

\subsection{Experimental settings}
\label{sec:setting}

\subsubsection{Datasets}
We evaluate \ours on a diverse collection of text-attributed graphs covering both homophilic and heterophilic settings. 
Specifically, we use six homophilic datasets from the CSTAG benchmark~\cite{cstag}, including Books-Children (Children), Books-History (History), Ele-Computers (Computers), Ele-Photo (Photo), Sports-Fitness (Fitness), and Ogbn-Arxiv-TA (Arxiv). 
We further include four heterophilic web-page graphs, namely Cornell, Texas, Wisconsin, and Washington~\cite{craven1998learning}. 
In all datasets, each node is associated with textual attributes and a category label, and the task is node classification. 
For Arxiv, we use the official public split provided by the OGB benchmark~\cite{hu2020ogb}. 
For all remaining datasets, following prior work~\cite{cstag,wang2025generalization}, we adopt a random 60\%/20\%/20\% split for train/validation/test. 
Dataset statistics are summarized in Table~\ref{tab:dataset}.

We briefly describe each dataset group below.

\begin{itemize}
    \item \textbf{Ogbn-Arxiv-TA (Arxiv)}~\cite{cstag} is derived from the citation dataset Ogbn-Arxiv~\cite{hu2020ogb}. 
    The task is to predict paper categories, formulated as a 40-class classification problem. 
    The textual attributes of each paper node are extracted from its title and abstract.

    \item \textbf{Books-Children and Books-History}~\cite{cstag} are extracted from the Amazon-Books dataset. 
    Books-Children contains items under the second-level category ``Children'', while Books-History contains items under the second-level category ``History''. 
    Nodes represent books, and edges indicate frequent co-purchase or co-viewing relations. 
    The labels correspond to third-level book categories, and the title and description of each book are used as textual attributes.

    \item \textbf{Ele-Computers and Ele-Photo}~\cite{cstag} are extracted from the Amazon-Electronics dataset. 
    Ele-Computers contains products under the second-level category ``Computers'', while Ele-Photo contains products under the second-level category ``Photo''. 
    Nodes represent electronics-related products, and edges indicate frequent co-purchase or co-viewing relations.

    \item \textbf{Sports-Fitness}~\cite{cstag} is extracted from the Amazon-Sports dataset and consists of items under the second-level category ``Fitness''. 
    Nodes represent fitness-related products, and edges indicate frequent co-purchase or co-viewing relations.

    \item \textbf{Cornell, Texas, Wisconsin, and Washington}~\cite{craven1998learning} are university web-page graphs. 
    Nodes represent web pages, edges represent hyperlinks, and textual features are derived from page content. 
    Each node is classified into one of several page types, such as student, faculty, or department.
\end{itemize}

\begin{table}[t]
  \caption{Datasets statistics. Both homophilic and heterophilic graphs are included in our experiments.}
  \label{tab:dataset}
  \centering
  \small
  \resizebox{\linewidth}{!}
  {
    \begin{tabular}{lccccc}
      \Xhline{1.2pt}
      \rowcolor{tableheadcolor} \textbf{Dataset} & \textbf{\#Nodes} & \textbf{\#Edges} & \textbf{\#Classes} & \textbf{Type}\\
      \midrule
      Children   & 76,875           & 1,554,578        & 24          & Homophilic       \\
      \rowcolor{tablerowcolor} History    & 41,551           & 358,574          & 12          & Homophilic       \\
      Computers    & 87,229           & 721,081          & 10          & Homophilic       \\
      \rowcolor{tablerowcolor} Photo        & 48,362           & 500,928          & 12          & Homophilic       \\
      Fitness   & 173,055          & 1,773,500        & 13         & Homophilic        \\
      \rowcolor{tablerowcolor} Arxiv    & 169,343          & 1,166,243        & 40          & Homophilic       \\
      \midrule
Cornell & 191 & 292 & 5 & Heterophilic \\
\rowcolor{tablerowcolor} Texas & 187 & 310 & 5 & Heterophilic \\
Wisconsin & 265 & 510 & 5 & Heterophilic \\
\rowcolor{tablerowcolor} Washington & 229 & 394 & 5 & Heterophilic \\       
      \Xhline{1.2pt}
    \end{tabular}
    
  }
\end{table}

\subsubsection{Baselines}
We compare \ours with a broad range of baselines covering graph neural networks, self-supervised graph learning methods, LLM-based predictors, and graph-oriented LLM methods. 
Specifically, the baselines fall into four categories. 
\textit{$\blacktriangleright$ Standard GNNs}: GCN~\cite{gcn}, GAT~\cite{gat}, GraphSAGE~\cite{graphsage}, and DGT~\cite{DGT}; 
\textit{$\blacktriangleright$ Graph contrastive learning (GCL) methods}: DGI~\cite{dgi}, BGRL~\cite{bgrl}, MaskGAE~\cite{maskgae}, and GraphMAE~\cite{graphmae}; 
\textit{$\blacktriangleright$ Vanilla LLMs}: \textsc{Qwen2.5-7B-Instruct}, \textsc{Llama-3.1-8B-Instruct}, \textsc{Qwen3-8B}, \textsc{GPT-4o}, and \textsc{DeepSeek-V4}; 
\textit{$\blacktriangleright$ Graph LLMs}: GLEM~\cite{glem}, TAPE~\cite{tape}, LlaGA~\cite{llaga}, OFA~\cite{ofa}, GOFA~\cite{gofa}, LLM-BP~\cite{wang2025generalization}, AuGLM~\cite{auglm}, and DENSE~\cite{bundle}. 
In adversarial robustness experiments, we additionally compare with robust GNNs, including RGCN~\cite{robustgcn}, MedianGCN~\cite{mediangcn}, SimPGCN~\cite{simpgcn}, and ProGNN~\cite{prognn}.
We briefly describe these baselines below.

\noindent\textbf{Standard GNNs.}
\begin{itemize}
    \item \textbf{GCN}~\cite{gcn}: A canonical graph convolutional model that mixes node features with a normalized adjacency operator. 
    \item \textbf{GAT}~\cite{gat}: Replaces fixed neighbor weights with learned attention coefficients, so different neighbors contribute unequally to the update. 
    \item \textbf{GraphSAGE}~\cite{graphsage}: An inductive GNN that samples neighbors and aggregates them with simple set functions, such as mean aggregation, followed by an MLP.
    \item \textbf{DGT}~\cite{DGT}: A sparse graph Transformer that dynamically selects structurally and semantically relevant nodes and applies sparse attention with learnable Katz positional encodings for efficient large-scale graph representation learning.
\end{itemize}

\noindent\textbf{Graph contrastive learning methods.}
\begin{itemize}
    \item \textbf{DGI}~\cite{dgi}: A self-supervised method that trains an encoder by maximizing mutual information between local patch embeddings and a global graph summary, using corrupted graphs as negatives. It learns representations without labels and then fits a downstream classifier.
    \item \textbf{BGRL}~\cite{bgrl}: A negative-free bootstrap approach with an online network predicting the target network's representations under graph augmentations. It stabilizes training via a stop-gradient target, similar in spirit to BYOL on graphs.
    \item \textbf{MaskGAE}~\cite{maskgae}: A masked graph autoencoding framework that hides parts of the graph structure and trains an encoder--decoder to reconstruct them. The masking objective encourages representations that capture both semantics and topology.
    \item \textbf{GraphMAE}~\cite{graphmae}: A masked autoencoder that reconstructs masked node attributes from partially observed graphs via an encoder--decoder pipeline. It is commonly used as self-supervised pretraining before supervised evaluation.
\end{itemize}

\noindent\textbf{Vanilla LLMs.}
We include \textsc{Qwen2.5-7B-Instruct}, \textsc{Llama-3.1-8B-Instruct}, \textsc{Qwen3-8B}, \textsc{GPT-4o}, and \textsc{DeepSeek-V4} as zero-shot or prompt-based LLM baselines without task-specific fine-tuning. 
These baselines evaluate whether language models alone can perform node classification on text-attributed graphs without explicit graph modeling.

\noindent\textbf{Graph LLMs.}
\begin{itemize}
    \item \textbf{GLEM}~\cite{glem}: Couples a text encoder with a graph model and iteratively leverages text and structure signals, typically via pseudo-labeling or co-training style updates. It is a representative hybrid that explicitly fuses language and graph inductive biases.
    \item \textbf{TAPE}~\cite{tape}: Uses an LLM to produce task-aware textual signals, such as explanations or enriched descriptions, that are then distilled or encoded for learning on text-attributed graphs.
    \item \textbf{LlaGA}~\cite{llaga}: A graph-aware LLM framework that injects graph context into LLM-based prediction, for example by linearizing local neighborhoods into text, and adapts the model for graph tasks.
    \item \textbf{OFA}~\cite{ofa}: A general graph foundation-style baseline that targets broad task coverage with a unified backbone and lightweight task adaptation.
    \item \textbf{GOFA}~\cite{gofa}: A stronger generalization-oriented extension in the same spirit as OFA, emphasizing robustness and transfer across datasets and tasks.
    \item \textbf{LLM-BP}~\cite{wang2025generalization}: Treats LLM predictions as local evidence and refines them with a belief-propagation-style graph inference procedure.
    \item \textbf{AuGLM}~\cite{auglm}: An LLM-based node classifier that enriches input with topological and semantic retrieval and employs a lightweight GNN for candidate label pruning.
    \item \textbf{DENSE}~\cite{bundle}: A zero-shot LLM-on-graph baseline that reduces cost and noise by grouping or bundling nodes and contexts before querying the LLM.
\end{itemize}

\noindent\textbf{Robust GNN baselines.}
The following baselines are used specifically in adversarial robustness experiments.
\begin{itemize}
    \item \textbf{RGCN}~\cite{robustgcn}: A robustness-oriented GCN variant that reduces sensitivity to noisy or adversarial edges by explicitly modeling uncertainty or unreliable neighborhoods.
    \item \textbf{MedianGCN}~\cite{mediangcn}: Replaces mean aggregation with a coordinate-wise median operator to suppress outlier neighbor messages.
    \item \textbf{SimPGCN}~\cite{simpgcn}: Encourages similarity-preserving representations and typically reweights edges based on feature similarity, reducing the impact of adversarially injected edges.
    \item \textbf{ProGNN}~\cite{prognn}: Jointly learns a cleaned graph structure together with GNN parameters under structural constraints.
\end{itemize}

We next describe how \ours generates robustness-aware node embeddings, which serve as the basis for subsequent retrieval and aggregation.

\subsubsection{LLM-based Embedding Generation}
For LLM-based embedding generation, we use a robustness-aware prompt that asks the model to distinguish consistent neighbor evidence from potentially misleading or adversarial neighbor information. 
The prompt encourages the generated embedding to preserve the target node's core semantics while downweighting inconsistent neighbor signals. 
The full prompt is shown in Table~\ref{tab:prompts}.

\begin{table}[t]
  \centering
  \caption{Prompt used to generate robustness-aware text embeddings.}
  \label{tab:prompts}

  \rule{\linewidth}{1.5pt}
  \parbox{\linewidth}{
    \vspace{4pt}
    \textbf{Task}:\\
    You are analyzing a target node in a graph where each node has associated text and a class label. 
    You are also provided with its neighbors' texts and labels. 
    However, due to adversarial perturbations, some neighbors may intentionally provide misleading information.\\[2pt]

    Please follow these steps:\\
    1. Read the target node's text.\\
    2. Check the consistency of each neighbor's text and label with the target node's content.\\
    3. For consistent neighbors, softly incorporate their information; for conflicting ones, discard or downweight them.\\
    4. Generate a numerical embedding that reflects the target node's core semantics and is robust against inconsistent or adversarial neighbors. 
    The embedding should capture the semantic meaning of the text and be useful for similarity comparison, clustering, and retrieval tasks.\\[2pt]

    \textbf{Input Text}:\\
    \{input\_text\}\\[2pt]

    \textbf{Reference Labels}:\\
    \{reference\_labels\}
    \vspace{4pt}
  }
  \rule{\linewidth}{1.5pt}
\end{table}

\begin{table*}[t]
    \centering
    \caption{Node classification results (\%) on text-attributed graphs. The best and second-best performances are highlighted as \first{boldfaced} and \second{underlined}, respectively.}\label{tab:node_clas}
    \small
    \resizebox{\linewidth}{!}
    {
        \def\arraystretch{1.15}
        \begin{tabular}{lcccccccccc}
            \Xhline{1.2pt}
            \rowcolor{tableheadcolor}
                                                         & \textbf{Children}        & \textbf{History}         & \textbf{Photo}          & \textbf{Computers}       & \textbf{Fitness}         & \textbf{Arxiv}          & \textbf{Cornell}        & \textbf{Texas}          & \textbf{Wisc.}          & \textbf{Wash.}          \\
            \Xhline{1.2pt}

            MLP                                          & 51.2$_{\pm0.2}$          & 82.0$_{\pm0.9}$          & 60.7$_{\pm0.9}$         & 65.3$_{\pm0.4}$          & 87.2$_{\pm0.3}$          & 46.5$_{\pm0.8}$         & 49.5$_{\pm0.7}$         & 62.0$_{\pm1.8}$         & 46.5$_{\pm1.5}$         & 66.0$_{\pm0.9}$         \\
            \rowcolor{tablerowcolor} GCN                 & 56.3$_{\pm0.3}$          & 84.4$_{\pm0.3}$          & 84.0$_{\pm0.9}$         & 87.6$_{\pm0.4}$          & 92.1$_{\pm0.7}$          & 72.5$_{\pm0.6}$         & 52.6$_{\pm0.4}$         & 64.8$_{\pm1.5}$         & 49.0$_{\pm0.7}$         & 69.5$_{\pm0.5}$         \\
            GAT                                          & 56.4$_{\pm0.8}$          & 84.0$_{\pm0.6}$          & 85.3$_{\pm0.5}$         & 88.6$_{\pm0.2}$          & 92.5$_{\pm0.5}$          & 72.7$_{\pm0.8}$         & 53.2$_{\pm0.6}$         & 65.9$_{\pm1.4}$         & 52.3$_{\pm0.4}$         & 73.1$_{\pm0.7}$         \\
            \rowcolor{tablerowcolor} GraphSAGE           & \second{58.5$_{\pm0.1}$} & 86.1$_{\pm0.8}$          & 82.9$_{\pm0.7}$         & 87.9$_{\pm0.8}$          & 92.8$_{\pm0.5}$          & 72.9$_{\pm0.8}$         & 54.8$_{\pm0.5}$         & 64.2$_{\pm1.3}$         & 50.1$_{\pm1.0}$         & 71.4$_{\pm0.6}$         \\
            DGT                                   & 57.4$_{\pm0.7}$          & 86.2$_{\pm0.3}$          & 85.1$_{\pm0.5}$         & 87.3$_{\pm0.3}$          & 92.0$_{\pm0.5}$          & 74.6$_{\pm0.4}$         & 54.5$_{\pm0.3}$         & 67.3$_{\pm1.0}$         & 51.4$_{\pm1.2}$         & 71.5$_{\pm0.5}$         \\

            \Xhline{1.0pt}
            \rowcolor{tablerowcolor} DGI                 & 54.9$_{\pm0.5}$          & 84.9$_{\pm0.3}$          & 82.7$_{\pm0.6}$         & 86.9$_{\pm0.4}$          & 92.9$_{\pm0.7}$          & 72.1$_{\pm0.5}$         & 14.7$_{\pm0.9}$         & 11.2$_{\pm1.4}$         & 12.1$_{\pm0.7}$         & 21.0$_{\pm1.2}$         \\
            BGRL                                         & 55.9$_{\pm0.4}$          & 85.9$_{\pm0.6}$          & 83.7$_{\pm0.5}$         & 87.9$_{\pm0.2}$          & 93.9$_{\pm0.4}$          & 73.3$_{\pm0.7}$         & 16.5$_{\pm1.0}$         & 12.8$_{\pm1.7}$         & 13.6$_{\pm1.4}$         & 22.3$_{\pm1.1}$         \\
            \rowcolor{tablerowcolor} MaskGAE             & 57.1$_{\pm0.8}$          & 86.1$_{\pm0.9}$          & 83.8$_{\pm0.4}$         & \second{89.1$_{\pm0.3}$} & 93.2$_{\pm0.6}$          & 73.4$_{\pm0.5}$         & 24.9$_{\pm1.1}$         & 19.3$_{\pm2.0}$         & 24.0$_{\pm1.6}$         & 26.1$_{\pm1.4}$         \\
            GraphMAE                                     & 57.0$_{\pm0.6}$          & \second{86.2$_{\pm0.5}$} & 83.7$_{\pm0.7}$         & 87.9$_{\pm0.4}$          & \second{93.4$_{\pm0.5}$} & 73.3$_{\pm0.6}$         & 23.0$_{\pm0.7}$         & 17.7$_{\pm0.4}$         & 23.0$_{\pm1.7}$         & 24.9$_{\pm0.9}$         \\

            \Xhline{1.0pt}
            \textsc{Llama3.1-8B}                         & 20.9                     & 36.0                     & 47.8                    & 53.9                     & 43.2                     & 38.1                    & 40.2                    & 58.5                    & 50.0                    & 44.3                    \\
            \rowcolor{tablerowcolor} \textsc{Qwen2.5-7B} & 23.0                     & 19.2                     & 41.0                    & 52.4                     & 56.2                     & 34.3                    & 38.1                    & 56.3                    & 48.7                    & 42.0                    \\
            \textsc{Qwen3-8B} & 26.6                     & 29.4                     & 45.6                    & 55.3                    & 58.4                     & 38.5                    & 42.3                    & 59.6                    & 52.3                    & 47.5                    \\            
            \textsc{DeepSeek-V4}                         & 30.5                     & 61.4                     & 55.1                    & 56.3                     & 48.8                     & 45.9                    & 47.0                    & 63.4                    & 53.4                    & 49.6                    \\
            \rowcolor{tablerowcolor} \textsc{GPT-4o}     & 30.8                     & 53.3                     & 55.6                    & 60.3                     & 66.4                     & 46.8                    & 45.5                    & 63.1                    & 56.6                    & 48.9                    \\

            \Xhline{1.0pt}
                        \rowcolor{tablerowcolor} GLEM                & 52.1$_{\pm0.5}$                   & 84.0$_{\pm0.8}$                     & 79.8$_{\pm0.3}$                    & 80.2$_{\pm0.6}$                     & 88.9$_{\pm0.2}$                     & 74.7$_{\pm1.3}$                    & 57.8$_{\pm1.0}$                    & 70.6$_{\pm0.8}$                    & 63.2$_{\pm0.6}$                    & 68.9$_{\pm0.5}$                    \\
            TAPE                                         & 53.4$_{\pm0.7}$                     & 83.8$_{\pm0.6}$                     & 86.1$_{\pm0.5}$                    & 84.5$_{\pm0.2}$                     & 85.7$_{\pm0.7}$                     & \second{75.2$_{\pm0.3}$}           & 78.6$_{\pm0.6}$                    & 84.2$_{\pm0.9}$                    & 80.3$_{\pm1.1}$                    & 74.9$_{\pm1.2}$                    \\
            \rowcolor{tablerowcolor} LLaGA               & 53.2$_{\pm0.4}$                     & 84.9$_{\pm0.3}$                     & \second{87.1$_{\pm0.7}$}           & 88.0$_{\pm0.4}$                     & 92.0$_{\pm0.3}$                     & \second{75.2$_{\pm0.9}$}           & 82.6$_{\pm0.8}$                    & 85.5$_{\pm1.6}$                    & 85.1$_{\pm0.9}$                    & 70.5$_{\pm0.7}$                    \\
            OFA                                          & 50.1$_{\pm0.5}$                     & 78.3$_{\pm0.3}$                     & 82.9$_{\pm0.5}$                    & 81.5$_{\pm0.6}$                     & 87.2$_{\pm0.6}$                     & 69.8$_{\pm0.5}$                    & 75.8$_{\pm0.5}$                    & 81.8$_{\pm0.4}$                    & 74.8$_{\pm0.3}$                    & 76.0$_{\pm0.6}$                    \\
            \rowcolor{tablerowcolor} GOFA                & 52.2$_{\pm0.3}$                     & 76.3$_{\pm0.4}$                     & 82.4$_{\pm0.6}$                    & 80.9$_{\pm0.2}$                     & 85.7$_{\pm0.9}$                     & 69.0$_{\pm0.9}$                    & 39.5$_{\pm0.6}$                    & 38.4$_{\pm1.0}$                    & 32.5$_{\pm1.5}$                    & 31.0$_{\pm1.1}$                    \\
            LLM-BP                                       & 54.8$_{\pm0.2}$                     & 79.9$_{\pm0.7}$                     & 84.1$_{\pm0.3}$                    & 84.1$_{\pm0.6}$                     & 81.2$_{\pm0.5}$                     & 67.8$_{\pm0.8}$                    & 83.3$_{\pm0.2}$                    & 81.7$_{\pm1.4}$                    & 77.8$_{\pm1.3}$                    & 73.1$_{\pm1.5}$                    \\
            AuGLM                                       & 55.2$_{\pm0.6}$                     & 81.4$_{\pm0.2}$                     & 83.5$_{\pm0.7}$                    & 82.7$_{\pm0.8}$                     & 80.9$_{\pm0.6}$                     & 69.4$_{\pm0.4}$                    & 81.5$_{\pm0.7}$                    & 80.1$_{\pm0.8}$                    & 79.4$_{\pm1.2}$                    & 78.5$_{\pm0.8}$                    \\
            \rowcolor{tablerowcolor} DENSE               & 51.4$_{\pm0.7}$                     & 77.4$_{\pm0.3}$                    & 82.5$_{\pm0.8}$                   & 86.5$_{\pm0.6}$                     & 84.1$_{\pm0.4}$                     & 70.5$_{\pm0.9}$                    & \second{84.8$_{\pm1.1}$}           & \second{92.5$_{\pm1.5}$ }          & \second{87.2$_{\pm0.9}$}           & \second{81.7$_{\pm1.7}$}           \\

            \Xhline{1.0pt}
            \rowcolor{methodrowcolor} \ours (ours)              & \first{61.9$_{\pm0.7}$}  & \first{87.0$_{\pm0.5}$}  & \first{88.7$_{\pm0.4}$} & \first{91.2$_{\pm0.3}$}  & \first{94.2$_{\pm0.5}$}  & \first{75.4$_{\pm0.9}$} & \first{89.9$_{\pm0.2}$} & \first{97.4$_{\pm0.4}$} & \first{96.3$_{\pm0.3}$} & \first{82.6$_{\pm0.2}$} \\
            \Xhline{1.2pt}
        \end{tabular}
    }
\end{table*}

\subsubsection{Hyperparameters}
We use a unified hyperparameter setting across datasets for fair comparison. 
Specifically, hyperparameters are tuned on \textsc{Ogbn-Arxiv-TA} and then kept fixed for all other datasets. 
All trainable models are optimized with Adam~\cite{adam} for at most 500 epochs, using a learning rate of 0.001 and early stopping with a patience of 10 epochs. 
For \ours and neural baselines, we use a hidden dimension of 128 and adopt a two-layer architecture unless otherwise specified.

For the retrieval components in \ours, we retrieve $k_1=3$ labeled examples for label-aware text embedding and $k_2=5$ nodes for retrieval-augmented message aggregation; we also set the PPR teleport probability $\alpha=0.5$ and the semantic-structural fusion weight $\lambda=0.5$ throughout all experiments. 
Following prior work~\cite{bundle,wang2025generalization}, we use \texttt{bge-en-icl}~\cite{li2024makingtextembeddersfewshot} as the default text embedding model for dense retrieval.

\subsubsection{Implementation and Evaluation}
We implement \ours and all trainable baselines using PyTorch~\cite{pytorch} and PyTorch Geometric~\cite{pyg}.
Accuracy is used as the evaluation metric. 
For MLP, GNN-based and GraphLLM methods, we report the average accuracy and standard deviation over ten runs. 
For LLM-based baselines, due to computational cost, we report the average results over three runs to account for the variability introduced by few-shot sampling.

To avoid label leakage, \ours uses only training labels as label-aware guidance. 
For validation and test nodes, as well as nodes without labels, we replace the label input with a padding token so that no ground-truth label information is injected beyond the training set. 
All datasets used in the experiments are publicly available. 
Unless otherwise specified, all experiments are conducted on an NVIDIA RTX 4090 GPU with 24 GB memory. 
We use the Transformers~\cite{transformers} library to implement the LLM models.

\subsection{Main results}
The node classification results on both homophilic and heterophilic datasets are presented in Table~\ref{tab:node_clas}.
Several key observations emerge from the results.

\textbf{RAG is a strong alternative to message passing}:
        Traditional GNNs and GCL methods rely on explicit message passing and structural encoding to capture graph relationships. However, our results show that \ours, which replaces conventional message passing with retrieval-based augmentation, consistently outperforms these approaches.
        Although we do not explicitly use recursive message passing in the original graph structure, \ours surpasses standard GNNs and GCL methods by a large margin, demonstrating the advantage of leveraging textual information in conjunction with graph structures.

\textbf{Simple MLP with RAG outperforms graph LLMs}:
        \ours is a simple MLP network augmented with RAG mechanisms, making it more parameter-efficient compared to graph LLMs. Compared to graph LLMs such as TAPE and LLaGA, \ours achieves significant improvements, particularly on datasets with richer textual features and heterophilic settings. This result highlights the effectiveness of retrieval-based augmentation in enhancing model performance without the need for specialized graph-specific pretraining.

\textbf{Graph-aware adaptation is crucial for LLMs to perform well on graph tasks}:
        We observe that vanilla LLMs perform poorly without fine-tuning, emphasizing the necessity of graph-aware adaptations for learning on text-attributed graphs. While LLMs excel at unstructured tasks, their inability to inherently model graph topology leads to subpar performance when applied directly to node classification tasks. The retrieval-based augmentation in \ours addresses this limitation by dynamically providing relevant contextual information, enabling LLMs to reason over graph data more effectively.




\begin{figure*}[t]
  \centering
  \includegraphics[width=\linewidth]{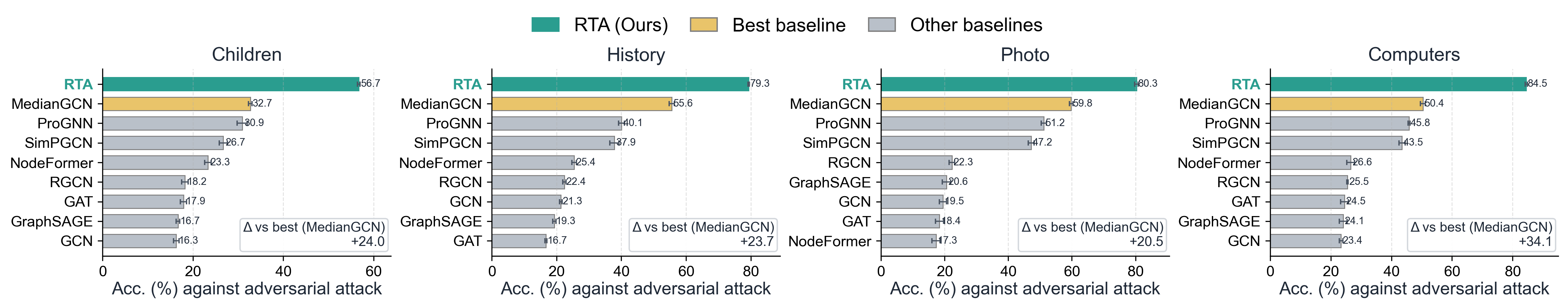}
  \caption{Performance comparison against graph adversarial attack \texttt{SGA}~\cite{sga}.}
  \label{fig:robustness}
\end{figure*}

\subsection{Robustness against adversarial attacks}
We evaluate adversarial robustness on four benchmark graphs (Children, History, Photo, and Computers) under the structure-perturbation attack \texttt{SGA}~\cite{sga}, which degrades node classification by adversarially modifying edges at test time. Figure~\ref{fig:robustness} reports accuracy under \texttt{SGA} (clean results omitted).

\ours consistently achieves the best performance across all datasets. Under \texttt{SGA} attack, \ours attains $56.7/79.3/80.3/84.5$ accuracy on Children/History/Photo/Computers, outperforming the strongest baseline by clear margins. In contrast, standard GNNs (e.g., GCN, GAT) collapse under attack, and robust defenses (e.g., MedianGCN) improve stability but remain substantially below \ours.
We attribute this robustness to retrieval-based message passing: instead of relying on a single, potentially poisoned local neighborhood, \ours retrieves task-relevant context that is less coupled to any specific adjacency and provides redundant evidence from multiple supports, leading to more stable predictions under edge perturbations.

\begin{figure*}[t]
  \centering
  \subfloat[Children]{\includegraphics[width=0.24\linewidth]{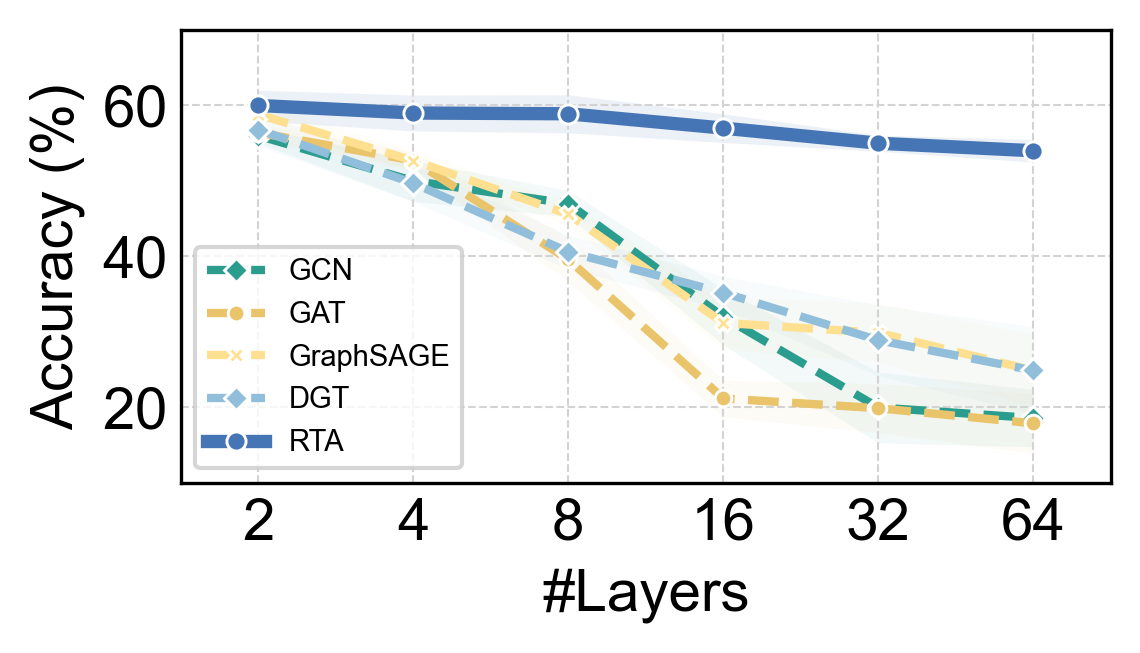}}
  \subfloat[History]{\includegraphics[width=0.24\linewidth]{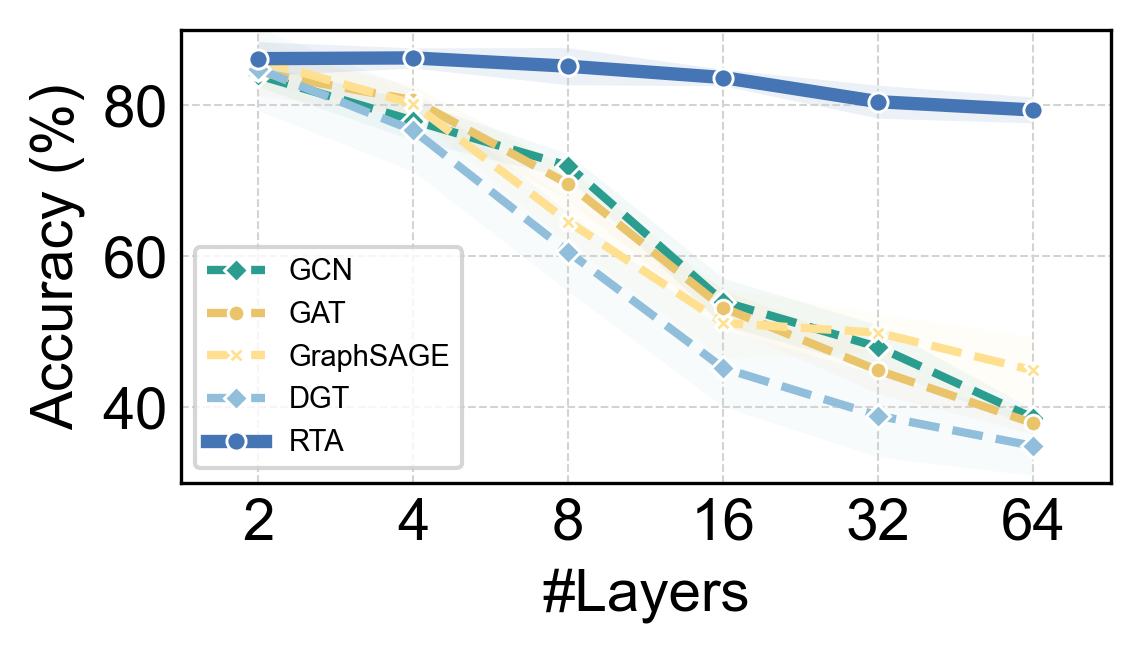}}
  \subfloat[Photo]{\includegraphics[width=0.24\linewidth]{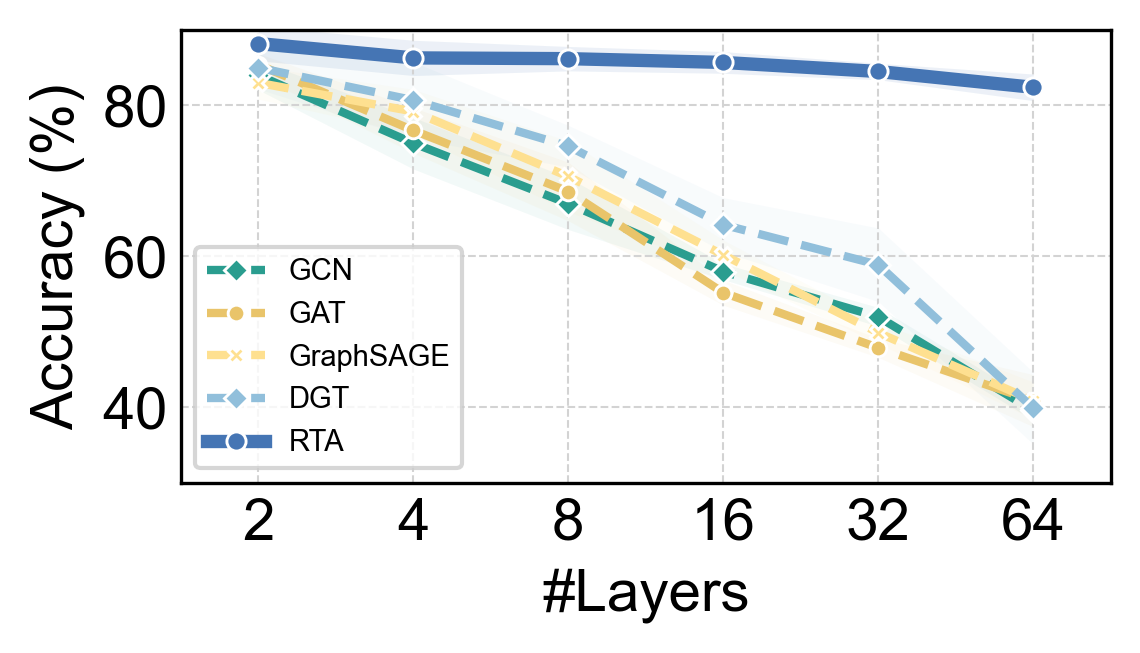}}
  \subfloat[Computers]{\includegraphics[width=0.24\linewidth]{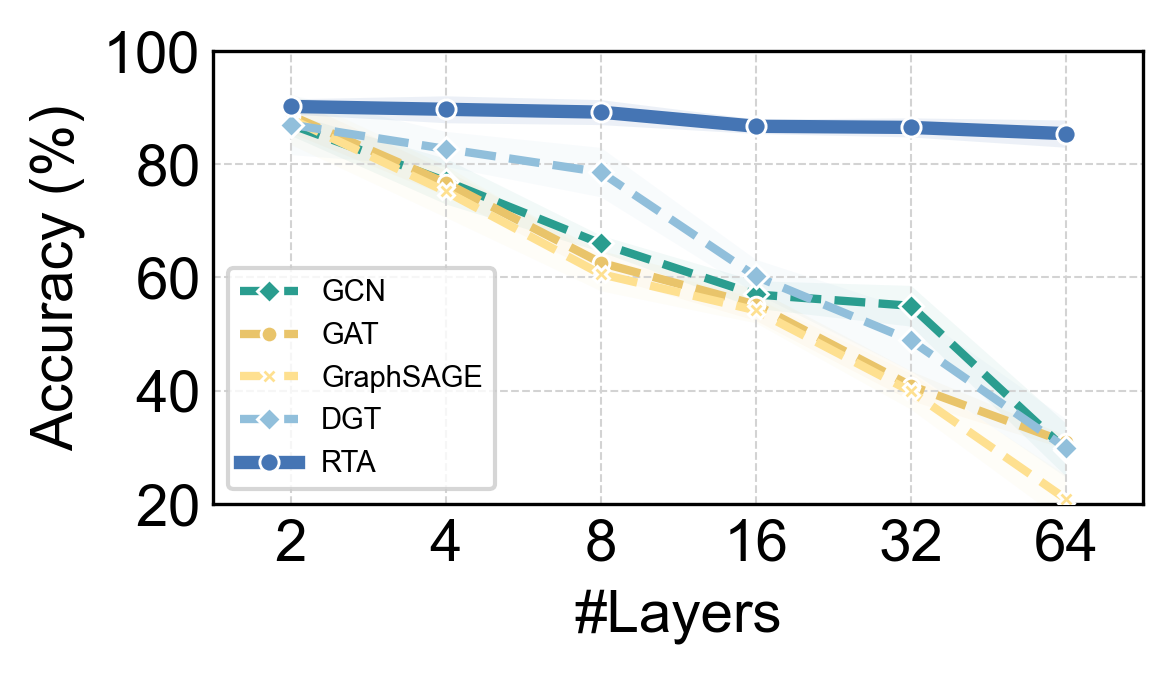}}
  \caption{Node classification performance with different network depth. Here we compare only with representative GNNs.}
  \label{fig:oversmoothing}
\end{figure*}

\subsection{Robustness against oversmoothing}
Although message passing enables GNNs to capture local graph structures, standard GNNs are known to suffer from the oversmoothing problem. Oversmoothing is a common phenomenon in GNNs, where increasing network depth leads to a deterioration in performance due to node representations becoming indistinguishable.
To further investigate whether \ours can mitigate the performance degradation of GNN methods in deep architectures, we conduct experiments on the Children and History datasets with varying numbers of layers/depths, $L \in \{2, 4, 8, 16, 32\}$. The results of different GNN methods are shown in Figure~\ref{fig:oversmoothing}.

As observed in Figure~\ref{fig:oversmoothing}, standard GNNs experience severe oversmoothing as network depth increases. The learned representations in these methods collapse at greater depths, leading to significantly degraded performance beyond four layers. In contrast, \ours demonstrates remarkable robustness against oversmoothing, scaling effectively to deeper architectures without substantial performance loss.
Notably, \ours achieves state-of-the-art performance across all depth levels. As the number of layers increases, \ours outperforms all baselines by increasingly larger margins. This highlights the effectiveness of retrieval-based augmentation in preserving representation quality and mitigating the limitations of traditional message-passing-based GNNs when scaling to deeper networks.

\subsection{Efficiency comparison}
\begin{figure}[h]
  \centering
  \includegraphics[height=0.25\linewidth]{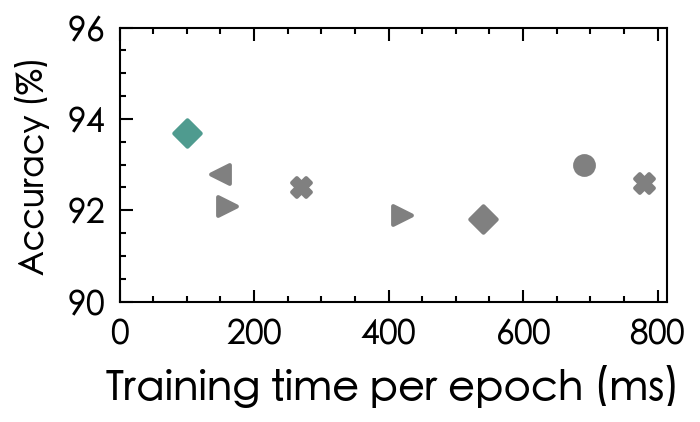}
  \includegraphics[height=0.25\linewidth]{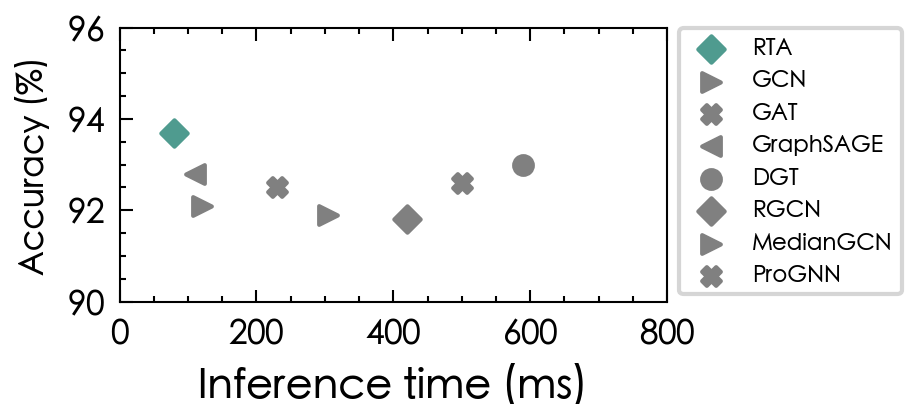}
  \caption{Comparison of training and inference time on Fitness dataset across different GNNs.}
  \label{fig:eff}
\end{figure}

\begin{figure*}[t]
  \centering

  \subfloat[Children]{
    \includegraphics[width=0.23\textwidth]{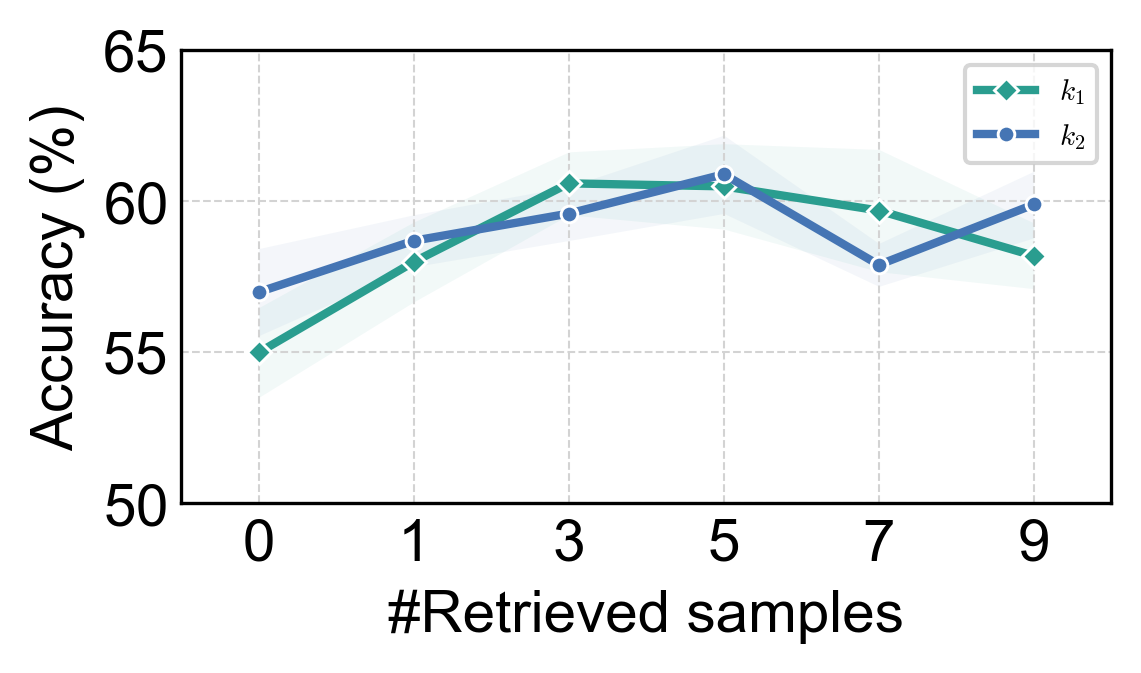}
  }
  \subfloat[History]{
    \includegraphics[width=0.23\textwidth]{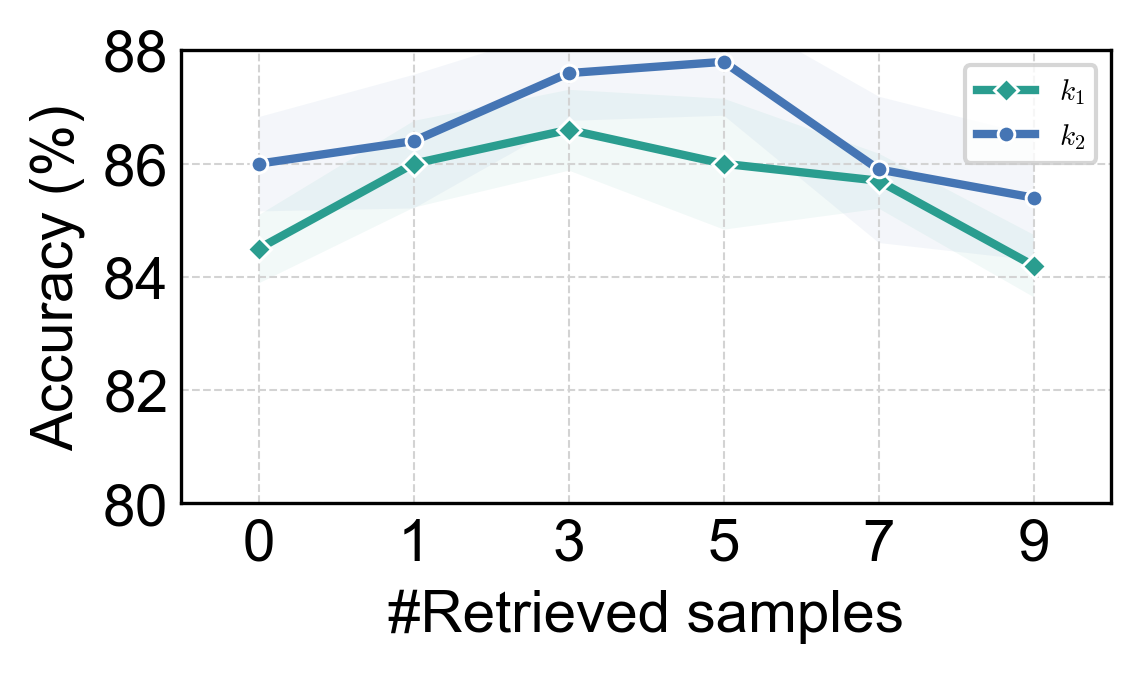}
  }
  \subfloat[Photo]{
    \includegraphics[width=0.23\textwidth]{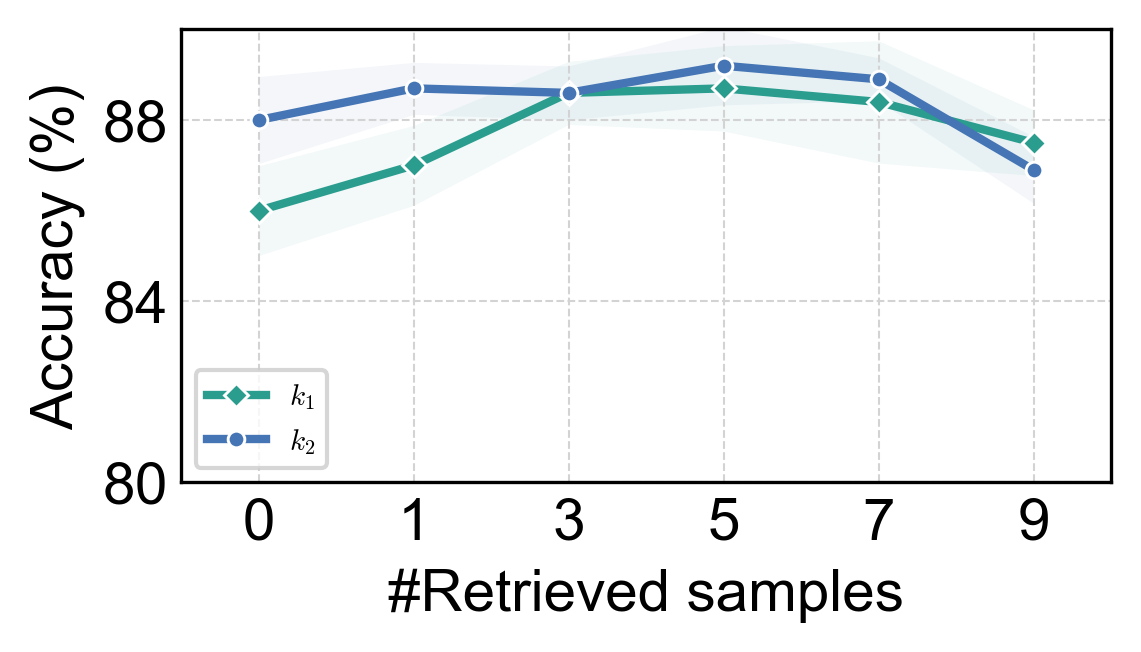}
  }
  \subfloat[Computer]{
    \includegraphics[width=0.23\textwidth]{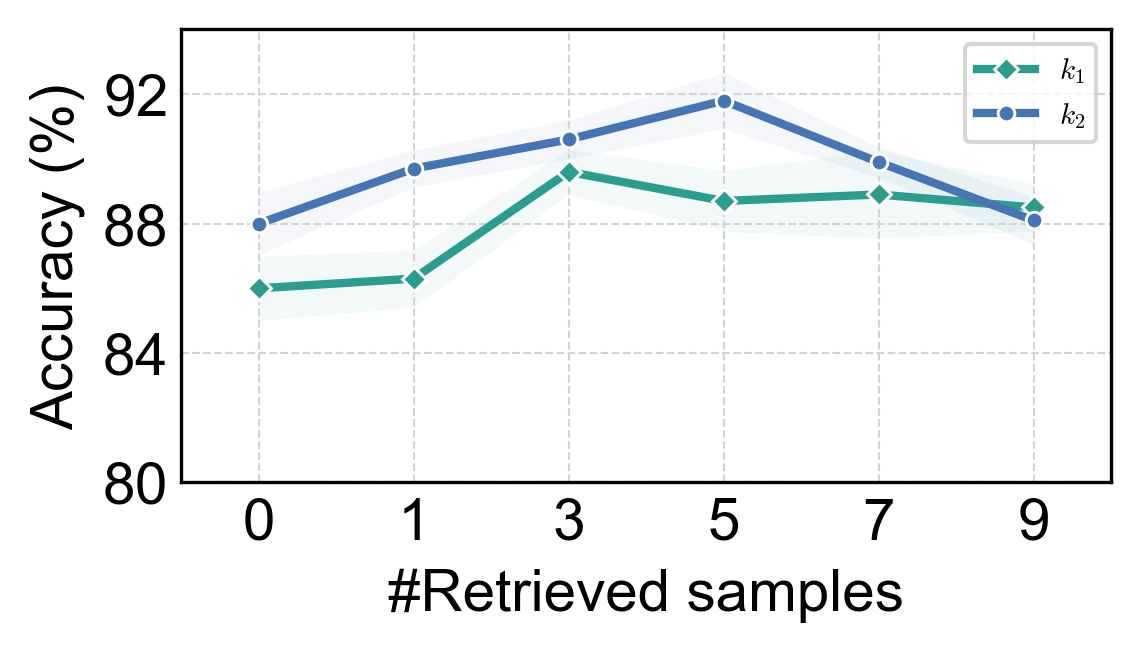}
  }

  \vspace{1mm}

  \subfloat[Cornell]{
    \includegraphics[width=0.23\textwidth]{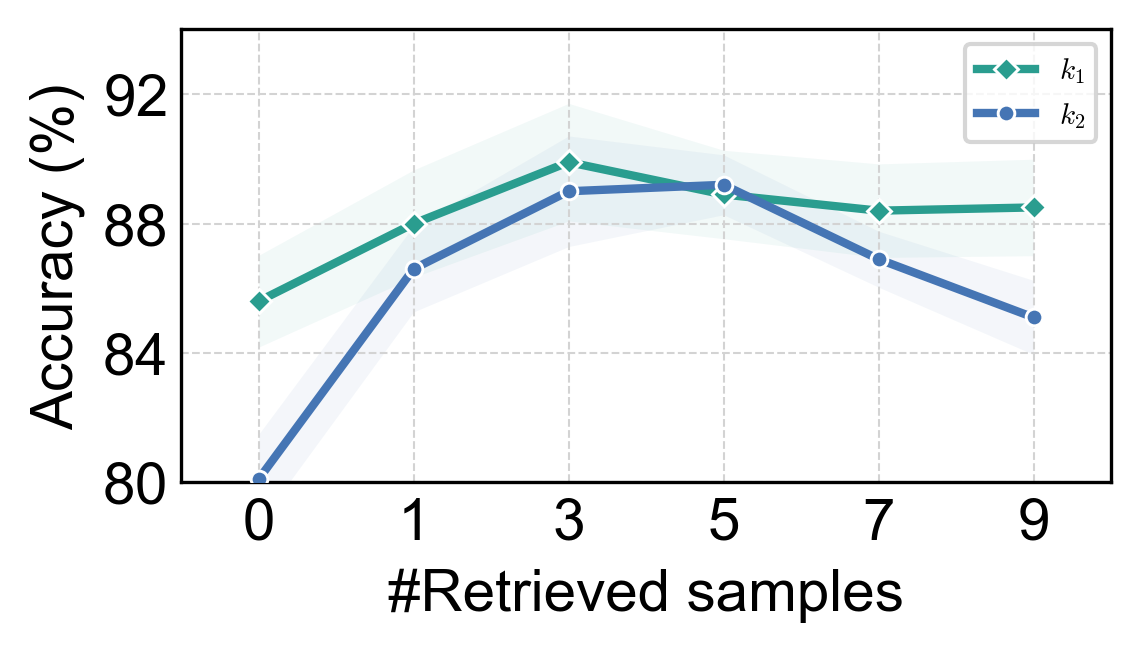}
  }
  \subfloat[Texas]{
    \includegraphics[width=0.23\textwidth]{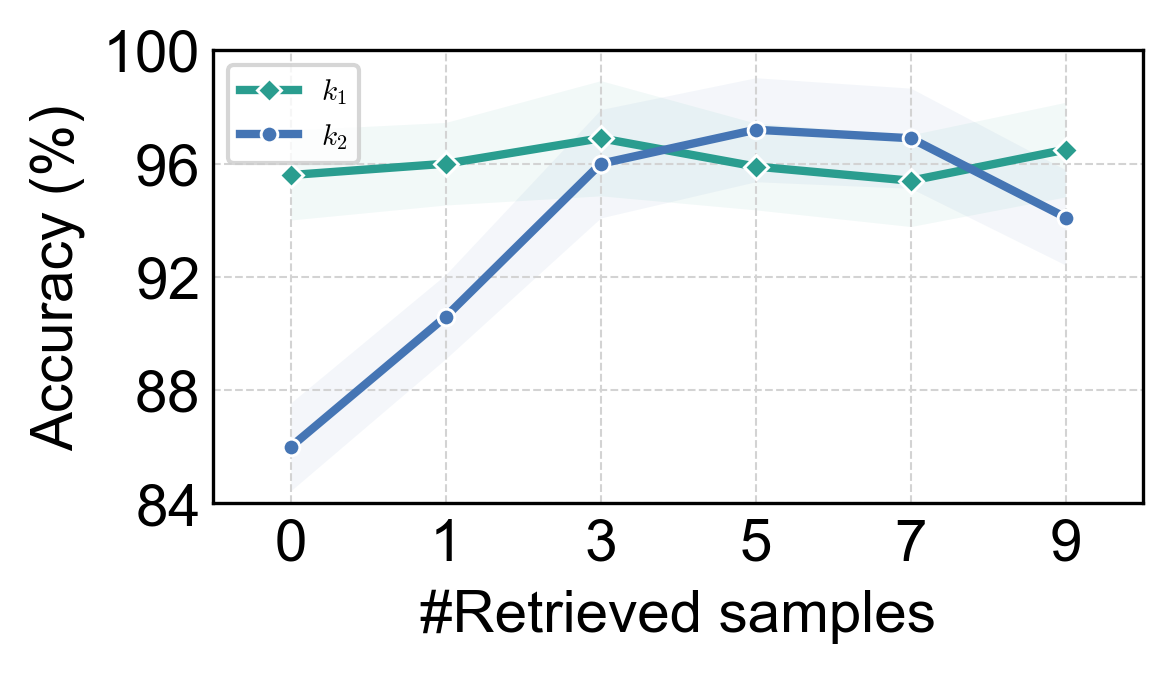}
  }
  \subfloat[Wisconsin]{
    \includegraphics[width=0.23\textwidth]{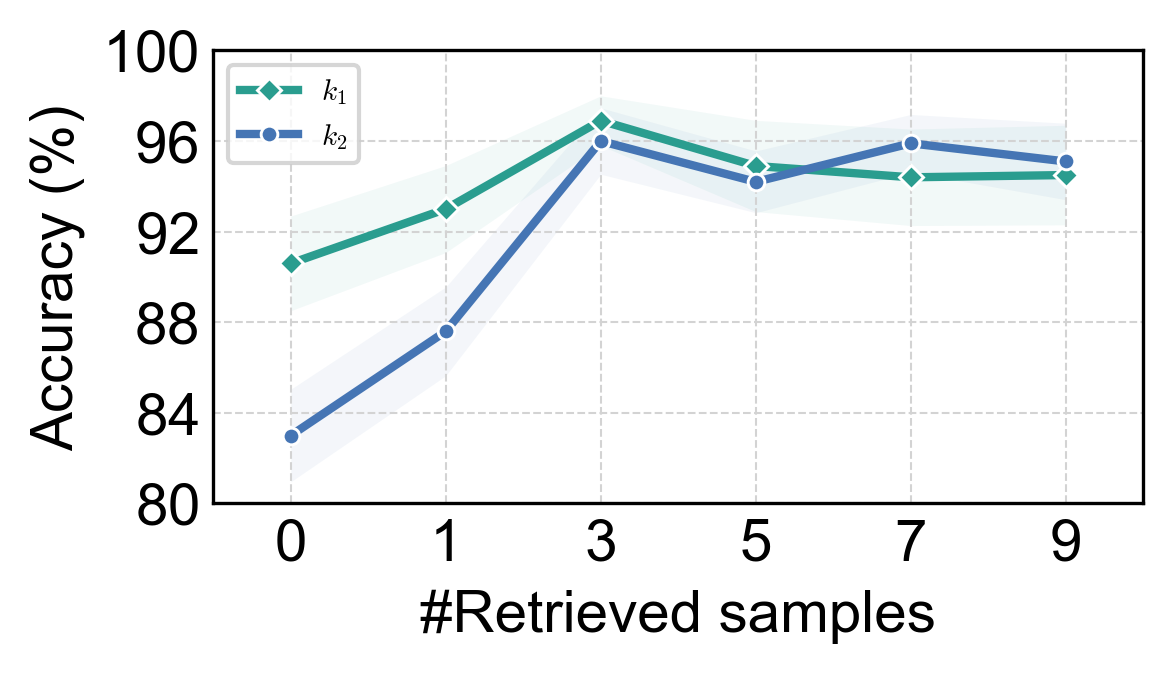}
  }
  \subfloat[Washington]{
    \includegraphics[width=0.23\textwidth]{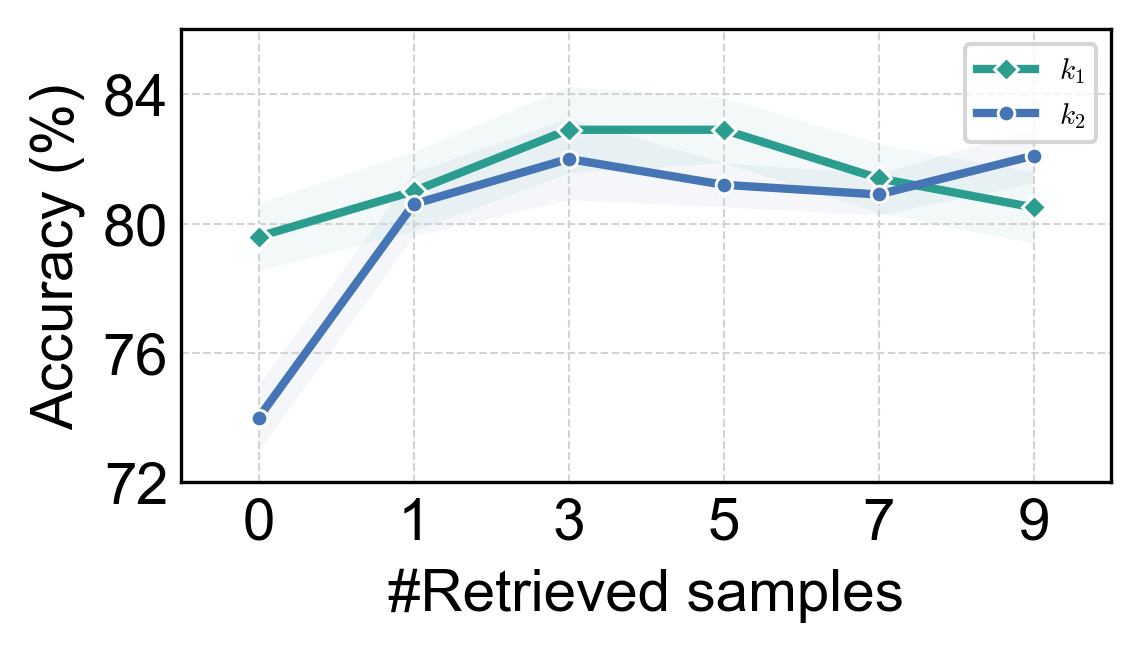}
  }

  \caption{Ablation results on the number of retrieved samples in text embedding ($k_1$) and message aggregation ($k_2$).}
  \label{fig:k}
\end{figure*}

\begin{table*}[h]
  \centering
  \caption{Comparisons between \ours and its variants. 
  L-I: label-aware contextual embedding; 
  L-II: label-aware aggregation.}
  \label{tab:abla}
  \small
  {
    \begin{tabular}{l|cccccccc}
      \Xhline{1.2pt}
      \rowcolor{tableheadcolor}
      & \textbf{Children} & \textbf{History} & \textbf{Photo}  
      & \textbf{Computers} & \textbf{Cornell} & \textbf{Texas} 
      & \textbf{Wisc.} & \textbf{Wash.} \\
      \Xhline{1.0pt}

      \ours$\backslash$ (L-I)     
      & 56.6$_{\pm0.3}$ & 85.4$_{\pm0.9}$ & 86.2$_{\pm0.8}$ & 86.9$_{\pm0.6}$    
      & 85.4$_{\pm0.4}$ & 95.1$_{\pm0.7}$ & 90.2$_{\pm0.6}$ & 80.7$_{\pm0.9}$ \\

      \ours$\backslash$ (L-II)    
      & 58.3$_{\pm0.6}$ & 86.6$_{\pm0.7}$ & 85.4$_{\pm0.5}$ & 88.8$_{\pm0.3}$ 
      & 79.4$_{\pm1.1}$ & 86.1$_{\pm1.5}$ & 83.3$_{\pm1.6}$ & 74.6$_{\pm1.4}$ \\

      \ours$\backslash$ (L-I\&II) 
      & 55.9$_{\pm0.7}$ & 84.7$_{\pm0.4}$ & 83.2$_{\pm0.3}$ & 85.4$_{\pm0.4}$  
      & 76.6$_{\pm1.4}$ & 68.6$_{\pm0.9}$ & 80.6$_{\pm0.8}$ & 69.6$_{\pm1.2}$ \\
      \hline
\ours$\backslash$ (semantic retrieval)
& 57.8$_{\pm0.7}$ & 84.9$_{\pm0.8}$ & 84.6$_{\pm0.7}$ & 87.1$_{\pm0.6}$
& 78.8$_{\pm1.3}$ & 84.7$_{\pm1.6}$ & 82.1$_{\pm1.5}$ & 73.8$_{\pm1.4}$ \\
\ours$\backslash$ (structural retrieval)
& 59.9$_{\pm0.6}$ & 85.9$_{\pm0.6}$ & 87.1$_{\pm0.6}$ & 89.8$_{\pm0.5}$
& 85.6$_{\pm0.9}$ & 93.2$_{\pm1.0}$ & 91.5$_{\pm1.1}$ & 79.8$_{\pm1.0}$ \\ 
      \hline
      \ours
      & \first{61.9$_{\pm0.7}$} & \first{87.0$_{\pm0.5}$} 
      & \first{88.7$_{\pm0.4}$} & \first{91.2$_{\pm0.3}$}
      & \first{89.9$_{\pm0.2}$} & \first{97.4$_{\pm0.4}$}
      & \first{96.3$_{\pm0.3}$} & \first{82.6$_{\pm0.2}$} \\
      \Xhline{1.2pt}
    \end{tabular}
  }
\end{table*}

\begin{table*}[t]
  \centering
    \caption{Retriever backbone ablation for \ours on eight node classification benchmarks. }\label{tab:retriever_ablation}

  \small
  {
    \begin{tabular}{lcccccccc}
      \Xhline{1.2pt}
      \rowcolor{tableheadcolor} \textbf{Retriever backbone}              & \textbf{Children}        & \textbf{History}         & \textbf{Photo}           & \textbf{Computers} & \textbf{Cornell}        & \textbf{Texas}          & \textbf{Wisc.}          & \textbf{Wash.}          \\
      \Xhline{1.0pt}
      \texttt{bert-base-uncased}      & 60.1$_{\pm0.4}$          & 85.7$_{\pm0.8}$          & 87.6$_{\pm0.6}$          & 90.3$_{\pm0.2}$       & 81.4$_{\pm1.6}$  & 80.4$_{\pm0.6}$ & 77.8$_{\pm2.0}$ & 68.6$_{\pm1.6}$    \\
      \texttt{all-mpnet-base-v2}      & 60.6$_{\pm0.3}$          & 86.2$_{\pm0.3}$          & 88.2$_{\pm0.3}$          & 90.8$_{\pm0.3}$   & 82.6$_{\pm1.2}$  & 84.9$_{\pm0.8}$ & 79.1$_{\pm1.4}$ & 69.6$_{\pm1.5}$       \\
      \texttt{bge-en-icl} & \first{61.9$_{\pm0.7}$}  & \first{87.0$_{\pm0.5}$}  & \first{88.7$_{\pm0.4}$} & \first{91.2$_{\pm0.3}$} & \first{89.9$_{\pm0.2}$} & \first{97.4$_{\pm0.4}$} & \first{96.3$_{\pm0.3}$} & \first{82.6$_{\pm0.2}$}\\
      \Xhline{1.2pt}
    \end{tabular}
  }
\end{table*}

We provide a comparison of training time between \ours and standard GNNs, including GCN, GAT, GraphSAGE, and DGT on the Fitness dataset - the largest dataset our experiments. The results in Figure~\ref{fig:eff} highlight that \ours achieves significantly higher computational efficiency while maintaining strong performance. Unlike traditional GNNs that require iterative message passing across multiple layers, \ours efficiently retrieves and incorporates contextual information through retrieval-augmented mechanisms, reducing redundant computations. Additionally, \ours avoids costly graph convolutions and neighbor aggregation steps, leading to faster convergence and lower training time. The efficiency of \ours is particularly advantageous for large-scale graph datasets, where standard GNNs often suffer from scalability issues due to high memory consumption and increased computational complexity. These results demonstrate that \ours provides a practical and scalable alternative to existing GNN models, offering a balance between computational efficiency and predictive accuracy.

\subsection{Ablation study}
In this subsection, we conduct ablation studies on eight representative benchmarks, including homophilic graph datasets Children, History, Photo, and Computers, and heterophilic graph datasets Cornell, Texas, Wisconsin, and Washington, to investigate the key factors that influence the performance of \ours.

\subsubsection{Impact of retrieved samples.}
We first conduct an ablation study on the number of retrieved text samples ($k_1$) and the number of neighbors used for message aggregation ($k_2$). The results are shown in Figure~\ref{fig:k}. Both $k_1$ and $k_2$ are important for the downstream performance of \ours. When $k_1=0$ or $k_2=0$, the method degenerates into a plain text-embedding baseline or a simple MLP without retrieved context or neighborhood aggregation.
Notably, relatively small values of $k_1$ and $k_2$ already yield strong performance, suggesting that a modest amount of high-relevance context is sufficient and helps avoid overwhelming the model. In contrast, increasing $k_1$ and $k_2$ beyond a certain point consistently hurts accuracy, likely because additional retrieved samples or neighbors become increasingly irrelevant or redundant and introduce noise. This degradation is more pronounced on heterophilic datasets, where indiscriminate neighborhood aggregation is more likely to mix semantically mismatched information across classes.
Overall, these results reveal a clear trade-off in which appropriately chosen $k_1$ and $k_2$ improve performance by injecting useful contextual signals.

\subsubsection{Ablation on label-aware and retrieval components.}
\ours couples retrieval with lightweight aggregation and incorporates label information into both the contextual embedding stage (L-I) and the aggregation stage (L-II).
We evaluate variants by removing L-I, L-II, or both components, with results reported in Table~\ref{tab:abla}.
Removing either component consistently degrades performance, confirming the benefit of label supervision at different stages of \ours.
In particular, removing L-I weakens the label-informed contextual representations used for subsequent retrieval, whereas removing L-II leads to more pronounced degradation on heterophilic graphs, where retrieved contexts are more likely to contain class-inconsistent information and label-aware aggregation helps suppress such noise.
Removing both components results in the largest overall performance drop, indicating that label guidance in contextual representation learning and retrieval aggregation provides complementary benefits.
We further examine the contribution of the two retrieval signals by individually removing semantic or structural retrieval.
Both variants perform consistently worse than the full model, demonstrating that semantic and structural evidence provide complementary information for context construction.
Semantic retrieval generally contributes more substantially, particularly on heterophilic graphs, where graph topology alone may provide less reliable task-relevant context.
Nevertheless, removing structural retrieval also leads to clear performance degradation, showing that PPR-based structural proximity captures useful high-order topological dependencies that are not fully reflected by embedding similarity.
Overall, jointly exploiting label-aware modeling together with semantic and structural retrieval enables \ours to construct more informative and task-relevant context sets.

\subsubsection{Ablation on the retrieval embedding model.}
In our main experiments, we follow prior work~\cite{bundle,wang2025generalization} and adopt \texttt{bge-en-icl}~\cite{li2024makingtextembeddersfewshot} as the default retrieval embedding model used to encode texts for dense retrieval.
To study the effect of embedding quality, we keep the retrieval pipeline fixed and only swap the text encoder among \texttt{bert-base-uncased}~\cite{bert}, \texttt{all-mpnet-base-v2}~\cite{sentence-bert}, and \texttt{bge-en-icl}.
As shown in Table~\ref{tab:retriever_ablation}, \texttt{bge-en-icl} achieves the best performance across all eight benchmarks.
The improvements are relatively modest on homophilic datasets, where neighborhood signals are already informative and retrieved texts mainly provide complementary evidence.
In contrast, on heterophilic datasets, \texttt{bge-en-icl} yields substantially larger gains, which suggests that stronger semantic embeddings are crucial when informative neighborhoods are harder to recover and naive aggregation is more likely to introduce mismatched context.
From a practical perspective, \texttt{all-mpnet-base-v2} offers a favorable accuracy--efficiency trade-off, while \texttt{bge-en-icl} is preferred when the additional compute is justified by the performance gains in challenging, especially heterophilic settings.

\section{Conclusion}
In this work, we establish a close connection between message-passing GNNs and retrieval-augmented generation.
We show that RAG-inspired designs can serve as a simplified yet effective alternative to conventional GNN architectures for text-attributed graphs.
By casting graph representation learning as a retrieval-augmented problem, \ours queries and aggregates task-relevant context from text-attributed nodes, which reduces the reliance on explicit message passing.
We further draw inspiration from label propagation and use training labels as guidance signals that shape both retrieval embeddings and neighborhood aggregation.
Theoretically, we formalize the relationship between RAG-based aggregation and standard message-passing GNNs, and we provide robustness guarantees for the optimization of label-aware retrieval.
Empirically, \ours achieves state-of-the-art results on a range of text-attributed graph benchmarks compared with strong GNN, LLM, and graph-LLM baselines, while demonstrating improved scalability and robustness under adversarial perturbations and mitigating oversmoothing.

\section*{Acknowledgments}
This work was supported by the New Generation Artificial Intelligence-National Science and Technology Major Project (No. 2025ZD0122701), the Young Scientists Fund (C) of the National Natural Science Foundation of China (No. 62602538), and the National Science Fund for Distinguished Young Scholars (No. 62525605).

\bibliography{reference}
\bibliographystyle{IEEEtran}

\appendix
\section{Theory and Proofs}
\label{sec:theory_and_proof}

This section provides theoretical support for the retrieval-augmented view developed in
Section~4. Rather than treating message passing as a graph-specific primitive,
we analyze it as \emph{prediction over a retrieved context set}. Concretely, for a query node $u$ we denote by
$\mathcal{R}_u^{(k)}$ its retrieved set (constructed by score-based Top-$k$ retrieval), and we view the downstream
computation as a permutation-invariant summarization of $\mathcal{R}_u^{(k)}$ followed by a parametric predictor.

We further consider the label-aware setting where a subset of nodes (e.g., the training set) can provide supervision as
retrieved \emph{label evidence}. Throughout, let $\mathcal{T}\subseteq \mathcal{V}$ denote the set of labeled nodes: for each $v\in\mathcal{T}$ we observe a (possibly noisy) label $\tilde y_v\in[C]$, while nodes outside $\mathcal{T}$ contribute only a fixed padding embedding (taken as $\mathbf{0}$ for simplicity). Under this notation, we establish two complementary results that align directly with the design of \ours: (i) an \emph{approximation} showing that Top-$k$ retrieval with mean aggregation provides a sparse approximation to softmax-attention message passing under a margin condition; and (ii) an \emph{optimization robustness} guarantee showing that retrieved-context supervision attenuates the gradient impact of mis-retrieved outliers.
Together, these statements clarify when retrieval-defined context sets can replace traditional message passing on fixed edges, and why injecting label evidence into retrieval and aggregation can improve stability without sacrificing robustness to residual retrieval noise.

\begin{assumption}[Bounded retrieved features]\label{ass:bounded}
Define normalized embeddings $\bar{\mathbf{x}}_u = \mathbf{x}_u/\|\mathbf{x}_u\|_2$, such that $\|\bar{\mathbf{x}}_u\|_2=1$.
Let $\omega:[C]\to\mathbb{R}^{d}$ be the label embedding map and define the label embedding used in aggregation as
\[
\mathbf{z}_v \;=\; \omega(\tilde y_v)\ \ \text{if } v\in \mathcal{T},\qquad \mathbf{z}_v=\mathbf{0}\ \ \text{otherwise}.
\]
Let the label-augmented neighbor feature be
\[
\mathbf{r}_v \;:=\; \mathbf{x}_v + \mathbf{z}_v \in \mathbb{R}^{d}.
\]
Assume coordinate-wise boundedness: for all $v$ and coordinates $j$,
\[
|(\mathbf{r}_v)_j|\le B.
\]
\end{assumption}

\begin{assumption}[Top-$k$ retrieval margin]\label{ass:margin}
Fix a node $u$. Let $s(u,v)$ be the retrieval score used in Eq.~(9) to construct $\mathcal{R}_u^{(k)}$.
In \ours, $s(u,v)=S(u,v)$ denotes the generalized retrieval score that combines semantic and structural relevance.
Let $v_{(1)},v_{(2)},\dots$ be nodes sorted by $s(u,\cdot)$ in descending order and define the top-$k$ margin
\[
\gamma_u \;=\; s\big(u,v_{(k)}\big) - s\big(u,v_{(k+1)}\big).
\]
Assume $\gamma_u>0$ for the nodes of interest.
\end{assumption}

\begin{theorem}[Formal statement of Theorem~1]
\label{thm:topk_as_attention}
Fix a node $u$ and let $s(u,v)$ be the retrieval score used in Eq.~(9).
Let $\mathcal{R}_u^{(k)}=\{v_{(1)},\ldots,v_{(k)}\}$ be the Top-$k$ set under $s(u,\cdot)$.
For any temperature $\tau>0$, define softmax attention weights over all candidates $v\neq u$ by
\[
a_{uv}(\tau)\;:=\;\frac{\exp\big(s(u,v)/\tau\big)}{\sum_{w\neq u}\exp\big(s(u,w)/\tau\big)}.
\]
Define the attention message and the Top-$k$ mean message as
\[
\mathbf{m}_u^{\mathrm{att}}(\tau)\;:=\;\sum_{v\neq u} a_{uv}(\tau)\,\mathbf{r}_v,
\qquad
\mathbf{m}_u^{\mathrm{mean}}\;:=\;\frac{1}{k}\sum_{v\in \mathcal{R}_u^{(k)}} \mathbf{r}_v,
\]
where $\mathbf{r}_v=\mathbf{x}_v+\mathbf{z}_v$ as in Assumption~\ref{ass:bounded}.
Let $\pi_u(\tau):=\sum_{v\notin \mathcal{R}_u^{(k)}} a_{uv}(\tau)$ be the attention mass outside the retrieved set, and let
\[
\delta_u \;:=\; s\big(u,v_{(1)}\big)-s\big(u,v_{(k)}\big)\ \ge 0
\]
be the score spread within the retrieved set.
Under Assumption~\ref{ass:bounded}, the following bound holds:
\[
\big\| \mathbf{m}_u^{\mathrm{att}}(\tau) - \mathbf{m}_u^{\mathrm{mean}}\big\|_\infty
\;\le\;
2B\,\pi_u(\tau)\;+\;B\big(\exp(\delta_u/\tau)-1\big).
\]
Moreover, under Assumption~\ref{ass:margin},
\[
\pi_u(\tau)\;\le\;(N-1-k)\exp\!\big(-\gamma_u/\tau\big),
\]
where $N=|\mathcal{V}|$.
\end{theorem}

\begin{proof}
Let $S=\mathcal{R}_u^{(k)}$. Define the renormalized weights on $S$ by
\[
\tilde a_{uv}(\tau)=\frac{a_{uv}(\tau)}{\sum_{w\in S}a_{uw}(\tau)}=\frac{a_{uv}(\tau)}{1-\pi_u(\tau)}\qquad (v\in S),
\]
and the truncated attention message $\tilde{\mathbf{m}}_u(\tau)=\sum_{v\in S}\tilde a_{uv}(\tau)\,\mathbf{r}_v$.

First, bound the leakage error:
\[
\mathbf{m}_u^{\mathrm{att}}(\tau)-\tilde{\mathbf{m}}_u(\tau)
=
\sum_{v\notin S} a_{uv}(\tau)\,\mathbf{r}_v
+
\sum_{v\in S}\Big(a_{uv}(\tau)-\tilde a_{uv}(\tau)\Big)\mathbf{r}_v.
\]
Since $\sum_{v\in S}|a_{uv}(\tau)-\tilde a_{uv}(\tau)|=\pi_u(\tau)$ and $\|\mathbf{r}_v\|_\infty\le B$ by Assumption~\ref{ass:bounded}, we obtain
\[
\|\mathbf{m}_u^{\mathrm{att}}(\tau)-\tilde{\mathbf{m}}_u(\tau)\|_\infty
\le
B\pi_u(\tau)+B\pi_u(\tau)=2B\pi_u(\tau).
\]

Second, bound the within-$S$ weighting error. By triangle inequality and $\|\mathbf{r}_v\|_\infty\le B$,
\[
\|\tilde{\mathbf{m}}_u(\tau)-\mathbf{m}_u^{\mathrm{mean}}\|_\infty
=
\Big\|\sum_{v\in S}\big(\tilde a_{uv}(\tau)-\tfrac{1}{k}\big)\mathbf{r}_v\Big\|_\infty
\le
B\sum_{v\in S}\Big|\tilde a_{uv}(\tau)-\tfrac{1}{k}\Big|.
\]
By definition of $\delta_u$, for all $v\in S$ we have
$s(u,v)\in[s(u,v_{(k)}),\,s(u,v_{(k)})+\delta_u]$.
Thus $\exp(s(u,v)/\tau)$ varies within a multiplicative factor $\exp(\delta_u/\tau)$ over $v\in S$, which implies
\[
\frac{1}{k\,\exp(\delta_u/\tau)}\;\le\;\tilde a_{uv}(\tau)\;\le\;\frac{\exp(\delta_u/\tau)}{k}\qquad \forall v\in S.
\]
Hence, $\big|\tilde a_{uv}(\tau)-\tfrac{1}{k}\big|\le \tfrac{\exp(\delta_u/\tau)-1}{k}$, and summing over $v\in S$ yields
\[
\sum_{v\in S}\Big|\tilde a_{uv}(\tau)-\tfrac{1}{k}\Big|
\le
\exp(\delta_u/\tau)-1.
\]
Therefore,
\[
\|\tilde{\mathbf{m}}_u(\tau)-\mathbf{m}_u^{\mathrm{mean}}\|_\infty
\le
B\big(\exp(\delta_u/\tau)-1\big).
\]
Combining the two bounds gives
\[
\|\mathbf{m}_u^{\mathrm{att}}(\tau)-\mathbf{m}_u^{\mathrm{mean}}\|_\infty
\le
2B\pi_u(\tau)+B\big(\exp(\delta_u/\tau)-1\big).
\]

Finally, bound $\pi_u(\tau)$ under Assumption~\ref{ass:margin}. For any $v\notin S$,
$s(u,v)\le s(u,v_{(k+1)})\le s(u,v_{(k)})-\gamma_u$.
Therefore,
\[
\sum_{v\notin S}\exp(s(u,v)/\tau)
\le
(N-1-k)\exp\!\big((s(u,v_{(k)})-\gamma_u)/\tau\big),
\]
while
\[
\sum_{v\in S}\exp(s(u,v)/\tau)\ge \exp\big(s(u,v_{(k)})/\tau\big).
\]
Thus,
\begin{align*}
\pi_u(\tau)
=
\frac{\sum_{v\notin S}\exp(s(u,v)/\tau)}{\sum_{w\neq u}\exp(s(u,w)/\tau)}
&\le
\frac{\sum_{v\notin S}\exp(s(u,v)/\tau)}{\sum_{v\in S}\exp(s(u,v)/\tau)}\\
&\le
(N-1-k)\exp(-\gamma_u/\tau).
\end{align*}
\end{proof}

\begin{remark}
Theorem~\ref{thm:topk_as_attention} justifies Top-$k$ retrieval as a principled sparsification of attention-based message passing. 
The theorem is stated for softmax-attention message passing because it provides a convenient envelope for many classical GNN layers. 
If the attention scores are constant on the context set, softmax attention reduces to uniform averaging, which matches the mean aggregator used in models such as GCN and GraphSAGE. 
Sum aggregation differs only by a degree-dependent scaling and can be viewed as the same update up to normalization. 
Therefore, the message-level conclusion remains informative for a broad class of message-passing GNNs, even when the final layer implementation does not explicitly use attention.
\end{remark}

\begin{remark}
\label{rem:margin_spread_mean}
Theorem~\ref{thm:topk_as_attention} separates two sources of approximation error between
$m_u^{\mathrm{att}}(\tau)$ and the Top-$k$ mean message $m_u^{\mathrm{mean}}$.
The first term $2B\,\pi_u(\tau)$ captures \emph{retrieval leakage}: even if the Top-$k$ set is fixed, attention still assigns
some probability mass to $v\notin\mathcal{R}_u^{(k)}$, which vanishes when the margin $\gamma_u$ is large (stable retrieval).
The second term $B(\exp(\delta_u/\tau)-1)$ captures \emph{within-set reweighting}: attention is generally non-uniform on
$\mathcal{R}_u^{(k)}$, and it approaches uniform (hence mean aggregation) when the score spread
$\delta_u = s(u,v_{(1)})-s(u,v_{(k)})$ is small, i.e., the retrieved items are similarly relevant.
This decomposition clarifies why a retrieval-based graph learner can replace attention-style message passing:
good retrieval quality corresponds to simultaneously (i) suppressing leakage via a clear margin and (ii) keeping the retrieved
set ``coherent'' so that a simple permutation-invariant mean operator is an accurate surrogate for attention.
\end{remark}

\begin{theorem}[Restatement of Theorem~2]\label{thm:tolerance}
Let $\mathcal{R}_u$ be a retrieved set for some query node $u$, with prediction distribution
$p(\mathcal{R}_u)=(p_1,\ldots,p_C)$ induced by retrieved-context aggregation.
Consider an outlier node $o\in\mathcal{R}_u$ whose individual distribution is $p_o=(p_1',\ldots,p_C')$ and define
$m'=\arg\max_{c\in[C]} p_c'$.
Let $y_u$ be the ground-truth label of node $u$, treated as fixed during differentiation.
If $y_u\neq m'$ and $p_{m'}'\ge p_{m'}$, then
\[
0 \leq \frac{\partial \mathcal{L}(\mathcal{R}_u)}{\partial \ell'_{m'}} \leq
\frac{\partial \mathcal{L}(o)}{\partial \ell'_{m'}},
\]
where $\mathcal{L}(\mathcal{R}_u)=\mathrm{CE}(p(\mathcal{R}_u),y_u)$ and
$\mathcal{L}(o)=\frac{1}{|\mathcal{R}_u|}\mathrm{CE}(p_o,y_u)$, and $\ell'_{m'}$ denotes the outlier logit for class $m'$.
\end{theorem}

\begin{proof}
Since $y_u$ is the ground-truth label of the query node, it is fixed with respect to model parameters during differentiation.

Let $\mathbf{e}\in\{0,1\}^C$ be the one-hot vector of $y_u$ (i.e., $e_{y_u}=1$).
For each node $i\in\mathcal{R}_u$, let $\boldsymbol{\ell}_i=(\ell_{i,1},\ldots,\ell_{i,C})$ be its logits. Define the set-level logits by mean pooling:
\[
\ell_c=\frac{1}{|\mathcal{R}_u|}\sum_{i\in\mathcal{R}_u} \ell_{i,c},\qquad c=1,\ldots,C,
\]
and set $p(\mathcal{R}_u)=p=\mathrm{softmax}(\boldsymbol{\ell})$. For the outlier node $o$, denote its logits by
$\boldsymbol{\ell}_o=\boldsymbol{\ell}'=(\ell_1',\ldots,\ell_C')$ and its distribution by $p_o=p'=\mathrm{softmax}(\boldsymbol{\ell}')$.

Using $\frac{\partial \mathcal{L}(\mathcal{R}_u)}{\partial \ell_c}=p_c-e_c$ and $\frac{\partial \ell_c}{\partial \ell'_c}=\frac{1}{|\mathcal{R}_u|}$ for the outlier node,
the chain rule yields
\[
\frac{\partial \mathcal{L}(\mathcal{R}_u)}{\partial \ell'_c}
=
\frac{1}{|\mathcal{R}_u|}(p_c-e_c).
\]
Likewise, since $\nabla_{\boldsymbol{\ell}'}\,\mathrm{CE}(\mathrm{softmax}(\boldsymbol{\ell}'),y_u)=p'-\mathbf{e}$, we obtain
\[
\frac{\partial \mathcal{L}(o)}{\partial \ell'_c}
=
\frac{1}{|\mathcal{R}_u|}(p'_c-e_c).
\]

Evaluating at $c=m'$ and using $y_u\neq m'$ (hence $e_{m'}=0$), we have
\[
\frac{\partial \mathcal{L}(\mathcal{R}_u)}{\partial \ell'_{m'}}=\frac{1}{|\mathcal{R}_u|}p_{m'}\ge 0,
\qquad
\frac{\partial \mathcal{L}(o)}{\partial \ell'_{m'}}=\frac{1}{|\mathcal{R}_u|}p'_{m'}.
\]
Together with $p'_{m'}\ge p_{m'}$, this implies
\[
0\le \frac{\partial \mathcal{L}(\mathcal{R}_u)}{\partial \ell'_{m'}}
\le \frac{\partial \mathcal{L}(o)}{\partial \ell'_{m'}}.
\]
\end{proof}

\begin{remark}
\label{rem:tolerance}
Theorem~\ref{thm:tolerance} shows that when retrieval introduces an outlier node $o$ whose most confident class $m'$ conflicts with the ground-truth label $y_u$ and satisfies $p'_{m'}\ge p_{m'}$, the retrieved-context objective $\mathcal{L}(\mathcal{R}_u)$ is more tolerant compared to the individual supervision $\mathcal{L}(o)$, as evidenced by a smaller penalty imposed by the gradient.
\end{remark}

\begin{remark}
\label{rem:label_rag}
Theorem~\ref{thm:tolerance} further indicates that even under imperfect retrieval, outlier neighbors whose predictions conflict with the ground-truth
label have their optimization impact attenuated under retrieved-context supervision. Consequently, injecting label information
can improve semantic coherence and retrieval precision, while the retrieved-context objective maintains robustness to the remaining retrieval noise.
Together, these effects help explain why combining retrieval with label-aware aggregation can simultaneously enhance optimization stability and
downstream performance in graph representation learning.
\end{remark}

\end{document}